%% file: arxiv_version_v2.tex
\documentclass{article}
\usepackage{authblk}
\usepackage{arxiv}
\usepackage{times}
\usepackage{graphicx}
\usepackage{amsmath}
\usepackage{amssymb}
\usepackage{shuffle}
\usepackage{color}
\usepackage{multirow}
\usepackage[table]{xcolor}
\usepackage{float}
\restylefloat{table}
\usepackage{diagbox,amsfonts, caption,  amsthm, mathabx}
\usepackage[title]{appendix}

\usepackage{amssymb}% http://ctan.org/pkg/amssymb
\usepackage{pifont}% http://ctan.org/pkg/pifont
\newcommand{\xmark}{\ding{55}}%

\theoremstyle{plain}
\newtheorem{theorem}{Theorem}[section]

\newtheorem{lemma}{Lemma}[section]
\newtheorem{definition}{Definition}[section]
\newtheorem{example}{Example}[section]
\newtheorem{model}{Model}[section]
\newtheorem{remark}{Remark}[section]

\graphicspath{{Figures/}}
\usepackage[pagebackref=true,breaklinks=true,letterpaper=true,colorlinks,bookmarks=false]{hyperref}

\begin{document}

%%%%%%%%% TITLE
\title{Learning stochastic differential equations using RNN with log signature features}

\author{
Shujian Liao\\
Department of Mathematics,\\
University College London
\\
{\tt\small ucahiao@ucl.ac.uk}

\and
Terry Lyons\\
Mathematical Institute,\\
University of Oxford\\
{\tt\small terry.lyons@maths.ox.ac.uk} \thanks{TL is supported by the EPSRC under the program grant EP/S026347/1 and by the Alan Turing Institute under the EPSRC grant EP/N510129/1.}
\and
Weixin Yang\\
Mathematical Institute,\\
University of Oxford\\
{\tt\small weixin.yang@maths.ox.ac.uk}\thanks{WY is supported by Royal Society Newton International Fellowship.}
\and
Hao Ni\\
Department of Mathematics,\\
University College London\thanks{HN is supported by the EPSRC under the program grant EP/S026347/1 and by the Alan Turing Institute under the EPSRC grant EP/N510129/1.}
\\
{\tt\small h.ni@ucl.ac.uk}

}

%\email{}
\maketitle
%\thispagestyle{empty}

%%%%%%%%% ABSTRACT

\begin{abstract}

This paper contributes to the challenge of learning a function on streamed multimodal data through evaluation. The core of the result of our paper is the combination of two quite different approaches to this problem. One comes from the mathematically principled technology of signatures and log-signatures as representations for streamed data, while the other draws on the techniques of recurrent neural networks (RNN). The ability of the former to manage high sample rate streams and the latter to manage large scale nonlinear interactions allows hybrid algorithms that are easy to code, quicker to train, and of lower complexity for a given accuracy. 

We illustrate the approach by approximating the unknown functional as a controlled differential equation. Linear functionals on solutions of controlled differential equations are the natural universal class of functions on data streams. Following this approach, we propose a hybrid Logsig-RNN algorithm (Figure \ref{RNN_SDE}) that learns functionals on streamed data . By testing on various datasets, i.e. synthetic data, NTU RGB+D 120 skeletal action data, and Chalearn2013 gesture data, our algorithm achieves the outstanding accuracy with superior efficiency and robustness.
\end{abstract}
\section{Introduction}
%\subsection{Motivation}
The relationship between neural networks and differential equations is an active area of research (\cite{weinan2017proposal}, \cite{lu2017beyond}, \cite{chen2018neural}). For example, Funahashi et al. introduced the continuous recurrent neural network(RNN) \cite{funahashi1993approximation}; He et al. connect residual networks and discretized ODEs \cite{weinan2017proposal}. A typical continuous RNN has the form 
\begin{eqnarray}\label{eqn_ODE}
\dot{Y}_{t} = - \frac{Y_{t}}{\tau} + A \sigma( B Y_{t}) + I_{t}, 
\end{eqnarray}
where $I_{t}$ and $Y_{t}$ are an input and output at time $t$ respectively\footnote{$\tau$ is a constant, $A$ and $B$ are matrices and $\sigma$ is an activation function.}.
Rough Path Theory teaches us that is more robust to consider the differential equation of the type 
\begin{eqnarray}\label{Eqn_Control}
dY_{t} = V(Y_{t})dX_{t},
\end{eqnarray}
and replace $I$ as an input with its integral. We can rewrite \eqref{eqn_ODE} in this form by setting $X_{t} =(t, \int_{s=t_{0}}^{t}I_{s}ds) $ and $V(y, (t, x)) = -\frac{y}{\tau}t + A\sigma(By)t + x$. %the input $X_{t} = \int_{s=t_{0}}^{t}I_{s}ds$ and to rewrite the differential equation 
% and $V$ is linear w.r.t the first variable. This model implicitly assumes that the input path $(I_{t})_{t}$ is the derivative of certain differentiable process $(X_{t})_{t}$.  \\
%However, when $X_{t}$ is far from differentiable, \eqref{eqn_ODE} is not well defined and thus the continuous RNN fails in this case. Stochastic differential equations (SDEs) are a generalization of Equation \eqref{Eqn_Control}, which are often written in the form
This allows the input to be of a broader type, and $X$ need not even be differentiable for the equation to be well defined. $Y$ inherits its regularity from $X$; equations in this form admit uniform estimates when $X$ is a rough path (highly oscillatory and potentially non-differentiable). This reformulation provides a much broader class of mathematical models for functionals on streamed data, of which the continuous RNN is a special case. 

In \cite{lyons1998differential} Lyons gives a deterministic pathwise definition to Equation \eqref{Eqn_Control} driven by rough signals. This analysis applies to almost all paths of e.g. vectored valued Brownian motion, diffusion processes, and also to many processes outside the SDE case, paths rougher than semi-martingales. \cite{lyons1998differential} articulates that in order to control the solution to Equation \eqref{Eqn_Control}, it suffices to control the $p$-variation and the iterated integrals of $X$ (the \emph{signature} of $X$) up to degree $\lfloor p \rfloor$. Crucially these estimates allow $p>>1$ and allow accurate descriptions of $Y_t$ to emerge from the coarse global descriptions of $X$ and its oscillations given by the signature. The \emph{log-signature} carries the exactly the same information as the signature but is considerably more parsimonious; it is a second mathematically principled transformation, and like the signature, it is able to summarize and vectorize complex un-parameterized streams of multi-modal data effectively over a \emph{coarse} time scale with a \emph{low dimensional} representation.

One area where this has been worked out in detail is with the numerical analysis of stochastic differential equations (SDEs). The most effective high order numerical approximation schemes for SDEs show that describing a path through the log-signature enables one to effectively approximate the solution to the equation and any linear functional of that solution globally over interval the path is defined on, without further dependence on the fine details of the \emph{recurrent structure} of the streamed data. It leads (in what is known as the log-ode method) to produce a state-of-the-art discretization method of of Inhomogenous Geometric Brownian Motion (IGBM)\cite{foster2019optimal}.% \ref{JamesFoster2019Foster, J., Lyons, T. and Oberhauser, H., 2019. An optimal polynomial approximation of Brownian motion. arXiv preprint arXiv:1904.06998.}
 We exploit this understanding to propose a simply but surprisingly effective neural network module (Logsig-RNN) by blending the Log-signature (Sequence) Layer with the RNN layer (see Figure \ref{RNN_SDE}) as an universal model for functionals on un-parameterized (and potentially complex) streamed data.

\begin{figure}[!ht]
		 \centering
		\includegraphics[width= 0.5 \textwidth]{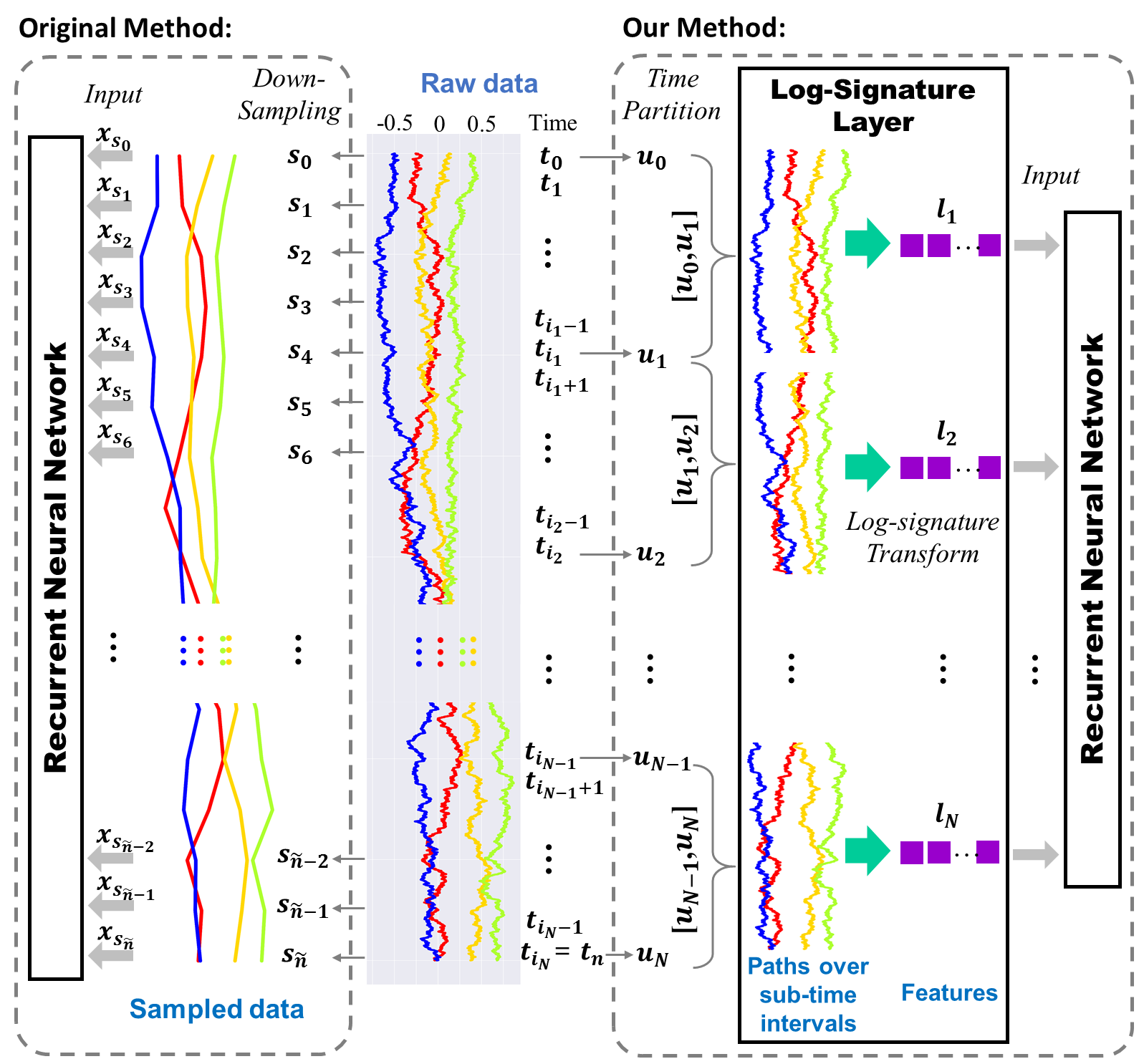}
	\caption{Comparison of Logsig-RNN and RNN.}\label{RNN_SDE}
\end{figure}
The Logsig-RNN network has the following advantages:
\begin{enumerate}
    \item \textbf{Time dimension reduction}: The Log-Signature Layer transforms a high frequency sampled time series to a sequence of the log-signatures over a potentially much coarser time partition. It reduces the time dimension of RNN significantly and thus speed up training time. 
    \item \textbf{High frequency and continuous data}: For the high frequency data case, the RNN type approach suffers from severe limitations, when applied directly \cite{donoho2000high}. In this case, one has to down-sample the stream data to a coarser time grid to feed it into the RNN-type algorithm (Figure \ref{RNN_SDE} (Left)). It may miss the microscopic characteristic of the streamed data and render lower accuracy. The Logsig-RNN model can handle  such case or even continuous data streams very well (Figure \ref{RNN_SDE}). 
    \item \textbf{Robustness to missing data}. Compared with the signature feature set, the log-signature is a parsimonious representation of the signature \cite{reutenauer2003free}, and empirically proved more robust to missing data. We validate the robustness of the Logsig-RNN model on various datasets, which outperforms that of the RNN model significantly.
    \item \textbf{Highly Oscillatory Stream data}: There is a fundamental issue when one accesses a highly oscillatory stream through sampling. It is quite possible for two streams to have very different effects and yet have near identical values when sampled at very fine levels \cite{friz2015physical}. %Even in the limit, the sequence of time series obtained through finer sampling need not provide a sufficient statistic that could predict the effects of a Brownian motion in an interaction! Conversely it is quite possible for complex signals to have very different fine structure but almost identical effects. 
    Therefore, to model a functional on a general highly oscillatory stream, the RNN on the sampled stream data would be challenged, requiring huge amounts of augmentation, and very fine sampling to be effective. In contrast, the rough path theory shows that if one postulates the (log)-signature of streamed data up in advance, the Logsig-RNN model can be much effective.   
\end{enumerate}

In summary, the main contributions of the paper are listed as follows:
\begin{enumerate}
\item to introduce the Log-signature (Sequence) Layer as a transformation of sequential data, and outline its back-propagation through time algorithm. It is highlighted that the Log-signature Layer can be inserted between other neural network layers conveniently, not limited to the pre-defined feature extraction.  
\item to design the novel neural network model (Logsig-RNN model) by blending the log-signature layer with RNN (Section \ref{SectionInference}) and prove the universality of the Logsig-RNN model for the approximation any solution map to the SDEs (Theorem \ref{UniversalityLogsigRNN});
%\item to design the effective algorithm based on Logsig-RNN model to learn the unknown SDEs, which is 
\item to propose the PT-Logsig-RNN model by adding the linear project layer in front of the Logsig-RNN architecture to tackle the case for the high dimensional input path (Section \ref{Ch_LP_Logsig_RNN}). 
\item to apply the Logsig-RNN algorithm to both synthetic data and empirical data to demonstrate its superior accuracy, effectiveness and robustness (Section \ref{SectionNumerics}). We achieve the state-of-the-art classification accuracy $93.27\%$ on ChaLearn2013 gesture data by the PT-Logsig-RNN model (Section \ref{chalearn2013}).
\end{enumerate}

\subsection{Related Work}
\subsubsection{Learning SDEs}
SDEs of the form \eqref{Eqn_Control} are useful tools for modelling random phenomena and provide a general class of functionals on the path space. SDEs not only are commonly used as models for the time-evolving process of many physical, chemical and biological systems of interacting particles \cite{gardiner1985handbook}, but also are the foundational building blocks in the derivatives pricing theory, an area of huge financial impact (\cite{black1973pricing}, \cite{merton1973theory}, \cite{cox1976valuation}). Statistical inference for SDEs has rich literature due to the importance of research outcomes and applications (see \cite{bishwal2007parameter} for the survey and overview). Most of the research focuses on the parameter estimation of (model-specific) stochastic processes; in particular \cite{papavasiliou2011parameter} is the pioneering work for the parameter estimation for a general stochastic process, which goes beyond diffusion processes by matching expected signature of the solution process. However, in contrast to these work, our approach is non-parametric and is used to learn the solution map without any assumption on the distribution of the stochastic process.  

\subsubsection{Rough paths theory in machine learning} Recently the application of the rough path theory in machine learning has been an emerging and active research area. The empirical applications of the rough paths theory primarily focused on the \emph{signature} feature, which serves as an effective feature extraction, e.g. online handwritten Chinese character/text recognition(\cite{graham2013sparse}, \cite{xie2018learning}), action classification in videos \cite{yang2017leveraging}, and financial data analysis (\cite{gyurko2013extracting}, \cite{lyons2014feature}). In addition, those previous work mainly combine the signature with the convolutional neural network or fully connected neural network. To our best knowledge, the proposed method is the first of the kind to integrates the sequence of log signature with the RNN. The log-signature brings many benefits (Section 2.3). The log-signature has been used as a local feature descriptor for gesture \cite{li2017lpsnet} and action recognition \cite{yang2017leveraging}. These used cases are bespoke; in contrast, the proposed Logsig-RNN is a general method for sequential data with outstanding performance in various datasets (See Section \ref{SectionNumerics}). Moreover, we extend the work on the back-propagation algorithm of the log-signature transformation in \cite{reizenstein2018iisignature} to the \emph{sequence} of the log-signature. %More importantly, this is the first work to use the numerical approximation theory of SDEs to uncover the natural link between the RNN and log-signature, and validated the performance of the proposed algorithm on various datasets. 

\subsubsection{Time series modelling}
In \cite{levin2013learning} Levin et al. firstly proposed the signature of a path as the basis functions for a functional on the un-parameterized path space and suggested the first non-parametric model for time series modelling by combining signature feature and the linear model (Sig-OLR). However, Sig-OLR has the limitation of inefficient global approximation due to the instability of the polynomial extrapolation. Despite the successful empirical applications of the signature feature sets (\cite{graham2013sparse}, \cite{xie2018learning}, \cite{yang2017leveraging}), the theoretical question on which learning algorithms are most appropriate to be combined with the (log)-signature feature remains open. Our work is devoted to answering this question with both theoretical justification and promising numerical evidence. 

\subsubsection{Functional Data Analysis}
Learning a functional on streamed data falls under the category of the functional data analysis (FDA) \cite{muller2011functional}, which models data using functions or functional parameters and analyse data providing information about curves, surfaces or anything else varying over a continuum. The representation theory of the functional on functions plays an important role in FDA study. Functional principal components analysis (\cite{silverman1996smoothed}) is one of the main techniques of FDA to represent the function data, which express the function data as the linear coefficients of the basis functions(usually without taking into account the response variable corresponding to the function input). In contrast to it,  albeit taking the functional view of sequential data, our approach focuses on the representation of the path in terms of its effect (functional on the path, i.e . the solution of the controlled differential equation driven by the path).

\section{The Log-Signature of a Path}\label{SectionPreliminaries}
Consider a continuous time series $x$ over the time interval $J:=[S, T]$ built at some very fine scale out of time stamped
values $x^{\hat{\mathcal{D}}} = [x_{t_{1}}, x_{t_{2}}, \cdots, x_{t_{n}}]$, where $\hat{\mathcal{D}} = (t_{1}, \cdots, t_{n})$.  When $x$ is highly oscillatory, to well capture effects of $x$, classical approaches requires sampling $x$ at high frequency, or even collect all the ticks. In this section, we takes the functional view on $x^{\hat{\mathcal{D}}}$ by embedding it to a continuous path by interpolation for a unified treatment (See detailed discussion in Section 4 of \cite{levin2013learning}). We start with the introduction to $p$-variation to measure the roughness of a continuous path. Then we introduce a graded feature, so-called log-signature feature as an effective and high order summary of a path of finite $p$-variation over time intervals. It follows with the key properties of the log-signature in machine learning applications.

\subsection{A Path with finite $p$-variation}
Let $E:=\mathbb{R}^{d}$ and $X: J \rightarrow E$ be a continuous path endowed with a norm denoted by $\vert \cdot \vert$. To make precise about the class of paths we discuss throughout the paper, we introduce the $p$-variation as a measure of the roughness of the path.
\begin{definition} [$p$-Variation]\label{p_variation}
Let $p\geq 1$ be a real number. Let $X: J \rightarrow E$ be a continuous path. The $p$-variation of $X$ on the interval $J$ is defined by 
\begin{equation}
    \vert\vert X\vert\vert_{p,J}=\left[ \sup_{\mathcal{D}\subset J}\sum_{j=0}^{r-1}\left\vert X_{t_{j+1}}-X_{t_j}\right\vert^p\right]^{\frac{1}{p}},
\end{equation}
where the supremum is taken over any time partition of $J$, i.e. $\mathcal{D} = (t_{1}, t_{2}, \cdots, t_{r})$.  \footnote{
Let $J = [S, T]$ be a closed bounded interval. A time partition of $J$ is an increasing sequence of real numbers $\mathcal{D} = (t_{0}, t_{1}, \cdots, t_{r})$ such that $S = t_{0} < t_{1}< \cdots < t_{r} = T$. Let $\vert \mathcal{D} \vert$ denote the number of time points in $\mathcal{D}$, i.e. $\vert \mathcal{D} \vert = r+1$. $\Delta \mathcal{D}$ denotes the time mesh of $\mathcal{D}$, i.e. $\Delta  \mathcal{D}:=\overset{r-1}{\underset{i = 0}{\max}} (t_{i+1} - t_{i})$.}
\end{definition}
Let $V_{p}(J,E)$ denote the range of any continuous path mapping from $J$ to $E$ of finite $p$-variation. The larger $p$-variation is, the rougher a path is. The compactness of the time interval $J$ can't ensure the finite $1$-variation of a continuous path in general (See Example \ref{ex2_infinite_length}).

\begin{example}\label{ex2_infinite_length}
	A fractional Brownian motion (fBM) with Hurst parameter $H$ has sample paths of finite $p$-variation a.s. for $p > \frac{1}{H}$. The larger $H$ is, the rougher fBM sample path is. For example, Brownian motion is a fBM with $H = 0.5$. It has finite $(2+\varepsilon)$-variation a.s $\forall \varepsilon >0$, but it has infinite $p$-variation a.s.$\forall p \in [1, 2]$.  
\end{example}

For each $p \geq 1$, the $p$-variation norm of a path $X: J \rightarrow E$ of finite $p$-variation is denoted by $\vert\vert X \vert\vert_{p-var}$ and defined as follows:
\begin{eqnarray*}
	\vert \vert X \vert\vert_{p-var} = \vert\vert X \vert\vert_{p, J} + \sup_{t \in J} \vert\vert X_{t}\vert\vert.
\end{eqnarray*}

\subsection{The log-signature of a path}
We introduce the signature and the log signature of a path, which take value in the tensor algebra space denoted by $T((E))$ endowed with the tensor multiplication and componentwise addition\cite{RoughPaths}.

\begin{definition}[The Signature of a Path]
Let $J$ denote a compact interval and $X : J \rightarrow E$ be a continuous path with finite $p$-variation such that the following integration makes sense. Let $I=(i_{1}, i_{2}, \cdots, i_{n})$ be a multi-index of length $n$ where $i_{j} \in \{1, \cdots, d\}, \forall j \in \{1, \cdots, n\}$. Define the coordinate signature of the path $X_{J}$ associate with the index $I$ as follows:
\begin{equation*}
    X^{I}_{J} = \underset{\underset{u_{1}, \dots, u_{k} \in J}{u_{1} < \dots < u_{k}}} { \int \dots \int} dX_{u_{1}}^{(i_{1})} \otimes \dots \otimes dX_{u_{n}}^{(i_{n})}.
\end{equation*}
The signature of $X$ is defined as follows:
\begin{equation}
    S(X)_J = (1, \mathbf{X}_J^1, \dots, \mathbf{X}_J^k, \dots).
\end{equation}
where $\mathbf{X}_J^k  = (X^{I}_{J})_{I = (i_{1}, \cdots, i_{k})}, \forall k \geq 1$. Let $S_{k}(X)_{J}$ denote the truncated signature of $X$ of degree $k$, i.e.
\begin{equation}
    S_{k}(X)_J = (1, \mathbf{X}_J^1, \dots, \mathbf{X}_J^k).
\end{equation}
\end{definition}
\begin{example}
When $X$ is a multi-dimensional Brownian motion, the above integration can be defined in both Stratonovich and It\^o sense. It is because that the Brownian motion has samples of infinite $p$-variation for  $p \in [ 1, 2]$ a.s (\cite{RoughPaths}).
\end{example}

\begin{remark}[The Signature of Discrete Time Series] The discrete version of a path $x^{\mathcal{D}}$ is of finite $1$-variation. Thus the signature of $x^{\mathcal{D}}$ is well defined. \textbf{It is highlighted that $S(x^{\mathcal{D}})$ is NOT the collection of all the monomials of discrete time series!} The dimension of all monomials of $x^{\mathcal{D}}$ grow with $\vert \mathcal{D} \vert$, while the dimension of $S_{k}(x^{\mathcal{D}})$ is invariant to $\vert \mathcal{D} \vert$.
\end{remark}
The signature of a path has many good properties, which makes the signature an efficient representation of an un-parameterized path. We refer readers to \cite{levin2013learning} for an introduction to the signature feature in machine learning.

The logarithm of the element in $T((E))$ is defined similar to the power series of the logarithm of a real value except for the multiplication is understood in the tensor product sense. 
\begin{definition}[Logarithm map]\label{eqn_log}
	Let $a = (a_{0}, a_{1}, \cdots) \in T((E))$ be such that $a_{0} = 1$ and $t = a - 1$. Then the logarithm map denoted by $\log$ is defined as follows:
	\begin{eqnarray}
	\log(a) = \log(1+t) = \sum_{n = 1}^{\infty} \frac{(-1)^{n-1}}{n} t^{\otimes n}, \forall a \in T((E)).
	\end{eqnarray}
\end{definition}

\begin{lemma}
	The logarithm map is bijective on the domain $\{ a \in T((E))\vert a_{0} = 1\}$. %The inverse of the logarithm map is the exponential map.
\end{lemma}

\begin{definition}[The Log Signature of a Path] The log signature of path $X$ by $\log(S(X))$ is the logarithm of the signature of the path $X$, denoted by $lS(X)$. Let $lS_{k}(X)$ denote the truncated log signature of a path $X$ of degree $k$.
\end{definition}
\subsection{Properties of the log-signature}
We summarize the key properties of the log-signature, which make the log-signature as a principled and effective summary of streamed data over intervals. In addition, we highlight the comparison of properties of the log-signature and the signature. More details on properties of (log)-signature can be found in the appendix, and the illustrative examples of pendigit data is provieded by pendigit$\_$demo.ipdb.\footnote{ pendigit$\_$demo.ipdb can be found via the github link: \url{https://github.com/logsigRNN/learn_sde/blob/master/Pen-digit_learning/pendigit_demo.ipynb}.}
\subsubsection{Bijection between the signature and the log-signature}
As the logarithm map is bijective, there is one-to-one correspondence between the signature and the log-signature \cite{RoughPaths}. The statement is also true for the truncated signature and log-signature up to the same degree. For example, by projecting both sides of Eqn \eqref{eqn_log} to $E$, the first level of signature and log-signature are both increments of the path $X_{T} - X_{S}$. For the second level of the signature $(X^{(i,j)_{S,T}})_{i, j \in =1}^{d}$, it can be decomposed to its symmetric and anti-symmetric parts
\begin{equation}
    X^{(i, j)}_{S, T} = \frac{1}{2}X^{(i)}_{S, T} X^{(j)}_{S, T} + A^{(i,j)}_{S, T},
\end{equation}
where 
%\begin{eqnarray}
$A_{S,T}^{(i,j)} = \frac{1}{2}\left(X^{(1, 2)} - X^{(2, 1)}\right).$
%\end{eqnarray}
It is noted that $A^{(i, i)} = 0$ and $lS_{2} = A^{(1,2)}[e_{1}, e_{2}]$. This is an example to show that the signature and log signature(lower dimension) is a bijection up to degree $2$. 
\subsubsection{Dimension Reduction}
The log-signature is a parsimonious representation for the signature feature, which is of lower dimension compared with the signature feature in general. It be used for significant dimension reduction. Let us consider the linear subspace of $T((E))$ equipped with the Lie bracket operation $[., .]$, defined as follows:
\begin{eqnarray*}
[a, b] = a \otimes b - b \otimes a.
\end{eqnarray*}

\begin{theorem}\label{theorem_log_sig_Lie_Series}(Theorem 2.3, \cite{RoughPaths})
For any path $X$ of finite $1$-variation , there exist $\lambda_{i_{1}, \cdots, i_{n}}$ such that the log-signature of $X$ can be expressed in the following form\footnote{The equivalent statement is that the log signature is a Lie series, which is defined in Definition \ref{def_lie_seires}.}
\begin{eqnarray}
lS(X) = \sum_{i = 1}^{d} \lambda_{i}e_{i}+  \sum_{\substack{n \geq 2 \\e_{i_1}, \cdots, e_{i_n}\\ \in \{1, \cdots, d\}}} \lambda_{i_{1}, , \cdots, i_{n}}[e_{i_{1}}, [\cdots, [e_{n-1}, e_{n}]]].\nonumber
\end{eqnarray}
\end{theorem}
The above theorem shows that the dimension of the truncated log-signature is no greater than that of the truncated signature due to the linear dependence of $[e_{i_{1}}, [e_{i_{2}}\cdots, [e_{n-1}, e_{n}]]]$. For example, $[e_{i}, e_{j}] = -[e_{j}, e_{i}]$. The analytic formula for the dimension of the truncated log signature can be found in Theorem \ref{dim_hall}. 
%is given as follows:
%\begin{eqnarray}\label{sigExample}
%&&S_{2} = [1, {\color{red}-0.1176,0.1176,} {\color{blue}0.0069, 0.0138,-0.0277,0.0069}]\nonumber\\
%&&lS_{2}= [{\color{red}-0.1176, 0.1176,}, {\color{blue}0.02026}]
%\end{eqnarray}

\subsubsection{Invariance under time parameterization}
We say that a path $\tilde{X}: J \rightarrow E$ is the time re-parameterization of $X:J \rightarrow E$ if and only if there exists a non-decreasing surjection $\lambda: J \rightarrow J$ such that $\tilde{X}_{t} = X_{\lambda(t)}$, $\forall t \in J$. 
\begin{lemma}\label{LogSigTimeInvariance}
	Let $X \in V_{1}(J, E)$  and a path $\tilde{X}: J \rightarrow E$ is the time re-parameterization of $X$. Then 
	\begin{eqnarray}\label{sig_time_inv}
	lS(X)_{J} = lS(\tilde{X})_{J}.
	\end{eqnarray} 
\end{lemma}
It is an immediate consequence of the bijection between the signature and log-signature, and the invariance of the signature (Lemma \ref{SigTimeInvariance}).
Re-parameterizing a path does not change its (log)-signature. In Figure \ref{TimeParameterizationInvariance}, speed changes result in different time series representation but the same (log)signature feature. The (log)-signature feature can remove the redundancy caused by the speed of traversing the path, which brings massive dimension reduction benefit. 

\begin{figure}
	\centering
	\begin{minipage}[c]{0.22\textwidth}
		% \centering
		\includegraphics[width= \textwidth]{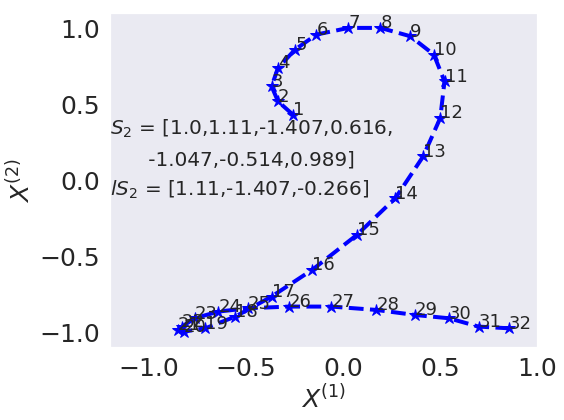}    
		%\caption{The sample of digit $8$}
		%\caption{Caption}
		%\label{fig:my_label}
	\end{minipage}
	\quad
	\begin{minipage}[c]{0.22\textwidth}
		% \centering
	%	\includegraphics[width= \textwidth]{no_param.png}
		\includegraphics[width= \textwidth]{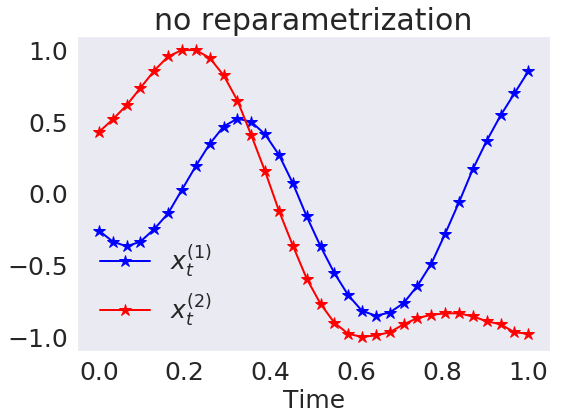}
		%\label{fig:my_label}
	\end{minipage}
	\quad
	\begin{minipage}[c]{0.22\textwidth}
		% \centering
	%	\includegraphics[width= \textwidth]{reparam1.png}
			\includegraphics[width= \textwidth]{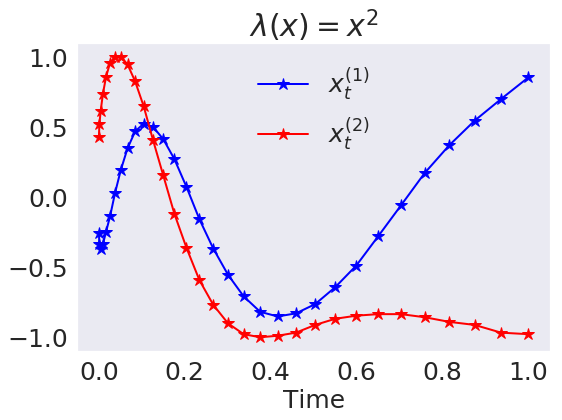}
		%\caption{Caption}
		%\label{fig:my_label}
	\end{minipage}
	\quad
	\begin{minipage}[c]{0.22\textwidth}
		% \centering
	%	\includegraphics[width= \textwidth]{reparam2.png}
			\includegraphics[width= \textwidth]{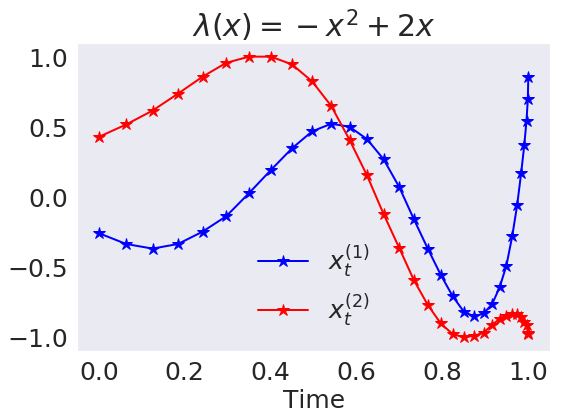}
		%\caption{Caption}
		%\label{fig:my_label}
	\end{minipage}
	\caption{The top-left figure represents the trajectory of the digit $2$, and the rest of figures plot the coordinates of the pen location via different speed respectively, which share the same signature and log signature given in the first subplot. }\label{TimeParameterizationInvariance}
\end{figure}

\subsubsection{Missing Data, variable length and unequally time spacing}
The (log)-signature feature set can both deal with time series of variable length and unequal time spacing. No matter the length of the time series is and how the time spacing is, the (log)-signature feature transformation provides a fixed dimension descriptor for any $d$-dimensional time series. Compared with the signature, the log-signature is empirically more robust to missing data (Figure \ref{sig_logsig_comparison_missing_data}). 
\begin{table}[H]
\centering
 \scalebox{0.85}{
\begin{tabular}{l|cc}
\hline
                                    & signature & log-signature \\
\hline
uniqueness of a path                  &    \checkmark        &    \checkmark             \\
invariance of time parameterization &      \checkmark       &     \checkmark            \\
the universality                     &    \checkmark         &       \xmark         \\
no redundancy                       &     \xmark      &       \checkmark         \\
\hline
\end{tabular}}
\caption{Comparison of Signature and Log-signature}\label{Tab_sig_logsig}
\end{table}

In contrast to signature, the log-signature does not have universality, and thus it needs be combined with non-linear models for learning. We summarize the comparison of the signature and log-signature in Table \ref{Tab_sig_logsig}. 
\section{Logsig-RNN Network}\label{section_logsig_RNN}
Consider a discrete $d$-dimensional time series $x^{\hat{\mathcal{D}}}  = (x_{t_{i}})_{i = 1}^{n}$  over time interval $J$. The lifted path associated with $x^{\hat{\mathcal{D}} }$ is the piecewise linear interpolation of $x^{\hat{\mathcal{D}}}$. Let $\mathcal{D}$ be a coarser time partition of $J$ such that $\mathcal{D}:= (u_{k})_{k = 0}^{N} \subset \hat{\mathcal{D}}$. 

\subsection{Log-Signature Layer}\label{BP_LogsignatureLayer}
We propose the Log-Signature (Sequence) Layer, which transforms an input $x^{\hat{\mathcal{D}}}$ to a sequence of the log signature of $x^{\hat{\mathcal{D}}}$ over a coarser time partition $\mathcal{D}$ .
\begin{definition}[Log-Signature (Sequence) Layer]\label{def_logsig_seq}
Given $\mathcal{D}$ and $\hat{\mathcal{D}}$, a Log-Signature Layer of degree $M$ is a mapping from $\mathbb{R}^{d \times n}$ to $\mathbb{R}^{d_{ls} \times N}$, which computes $(l_{k})_{k =0}^{N-1}$ as an output for any $x^{\mathcal{D}}$, where 
$l_{k}$ is the truncated log signature of $x$ over the time interval $[u_{k}, u_{k+1}]$ of degree $M$ as follows:
\begin{eqnarray}\label{def_ls}
l_{k}=lS_{M}(x)_{[u_{k}, u_{k+1}]},
\end{eqnarray}
where $k \in \{0, 1, \cdots N-1\}$ and $d_{ls}$ is the dimension of the truncated log-signature.
\end{definition}
It is noted that the Log-Signature Layer does not have any weights. In addition, the input dimension of Log-signature layer is $(d, n)$ and the output dimension is $(N, d_{ls})$ where $N \leq n$ and $d_{ls} \geq d$. The Log-Signature Layer potentially shrinks the time dimension effectively by using the more informative spatial features of a higher dimension.

\textbf{Backpropogation} Let us consider the derivative of the scalar function $F$ on $(l_{k})_{k = 1}^{N}$ with respect to path $x^{\hat{\mathcal{D}}}$, given the derivatives of $F$ with respect to $(l_{k})_{k = 1}^{N}$. By the Chain rule, it holds that
\begin{eqnarray}\label{BP_1}
\frac{\partial F((l_{1}, \cdots, l_{N}))}{\partial x_{t_{i}}} = \sum_{k = 1}^{N}\frac{\partial F(l_{1}, \cdots, l_{N})}{\partial l_{k}}\frac{\partial l_{k}}{\partial x_{t_{i}}}.
\end{eqnarray}
where $k \in \{1, \cdots, N\}$ and $i \in \{0, 1, \cdots, n\}$.

If $t_{i} \notin [u_{k-1}, u_{k}]$, $\frac{\partial l_{k}}{\partial x_{t_{i}}} = 0;$ otherwise $\frac{\partial l_{k}}{\partial x_{t_{i}}}$ is the derivative of the single log-signature $l_{k}$ with respect to path $x_{u_{k-1}, u_{k}}$  where $t_{i} \in \mathcal{D} \cap [u_{k-1}, u_{k}]$. The log signature $lS(x^{\hat{\mathcal{D}}})$ with respect to $x_{t_{i}}$ is proved differentiable and the algorithm of computing the derivatives is given in \cite{ReizensteinIhesis2018}, denoted by $\triangledown_{x_{t_{i}}} LS(x^{\hat{\mathcal{D}}})$. This is a special case for our log-signature layer when $N = 1$. In general, for any $N \in \mathbb{Z}^{+}$, it holds that $\forall  i \in \{0, 1, \cdots, n\}$ and $k \in \{1, \cdots, N\}$,
\begin{eqnarray}\label{logsigseq_eqn}
\frac{\partial l_{k}}{\partial x_{t_{i}}} = \mathbf{1}_{t_{i} \in [u_{k-1}, u_{k}]} \triangledown_{x_{t_{i}}} LS(x_{u_{k-1}, u_{k}}),
\end{eqnarray}
Thus the backpropogation algorithm of the Log-Signature Layer can be implemented using Equation \eqref{BP_1} and \eqref{logsigseq_eqn}.\footnote{In iisignature python package \cite{ReizensteinIisignature2018}, logsigbackprop(deriv, path, s, Method = None) returns the derivatives of some scalar function $F$ with respect to path, given the derivatives of $F$ with respect to logsig(path, s, methods). Our implementation of the back-propogation algorithm of the log-signature layer uses logigbackprop() provided in iisignature. }
\subsection{Logsig-RNN Network}\label{SectionInference}
Before proceeding to the Logsig-RNN network, we introduce the conventional recurrent neural network(RNN). RNN is composed with three types of layers, i.e. the input layer $(x_{t})_t$, the hidden layer $(h_{t})_t$ and the output layer $(o_{t})_{t}$. RNN takes the sequence of $d$-dimensional vectors $(x_{1}, x_{2}, \cdots, x_{T})$ as an input and compute the output $(o_{t})_{t = 1}^{T} \in \mathbb{R}^{e \times T}$ using Equation \eqref{rnn_eqn}:
\begin{eqnarray}\label{rnn_eqn}
h_t &=& \sigma(Ux_t + Wh_{t-1}), o_{t}= q(Vh_t),
\end{eqnarray}
where $U$, $W$ and $V$ are model parameters, and $\sigma$ and $q$ are two activation functions in the hidden layer and output layer respectively. Let $\mathcal{R}_{\sigma}((x_{t})_{t} \vert \Theta)$ denote the RNN model with $(x_{t})_{t}$ as the input, $\sigma$ and linear function $q$ as the activation functions of the hidden layer and output layer respectively and $\Theta:=\{U, W, V\}$ is the parameter set of the RNN model.

We propose the Logsig-RNN model by incorporating the log-signature layer to the RNN. We defer the motivation of the Logsig-RNN Network to next section. %Section \ref{sectionLearnSDEs}.  

\begin{model}[Logsig-RNN Network]\label{logsigrnn}
Given $\mathcal{D}:= (u_{k})_{k = 0}^{N}$, a Logsig-RNN network computes a mapping from an input path $x^{\mathcal{D}}$ to an output defined as follows:
\begin{itemize}
    \item Compute $(l_{k})_{k =0}^{N-1}$ as the output of the Log-Signature Layer of degree $M$ for an input $(x^{\hat{\mathcal{D}}})$ by Definition \ref{def_logsig_seq}.
 \item The output layer is computed by $\mathcal{R}_{\sigma}((l_{k})_{k = 0}^{N-1} \vert \Theta)$, where $\mathcal{R}_{\sigma}$ is a RNN network with certain activation function $\sigma$.
\end{itemize}
\end{model}
\begin{remark}[Link between RNN model and Logsig-RNN model]
For $M = 1$, Logsig-RNN network is reduced to the RNN model with $(x_{u_{k+1}} - x_{u_{k}} )_{k = 1}^{N}$ as an input. When $\mathcal{D}$ coincides with $\hat{\mathcal{D}}$, the Logsig-RNN Model is the RNN model with increment of raw data input. 
\end{remark}
\begin{remark}
The sampling time partition of raw data $\hat{\mathcal{D}}$ can be potentially much higher than $\mathcal{D}$ used in Logsig-RNN model. The higher frequency of input data would not increase the dimension of the log-signature layer, but it makes the computation of $l_{k}$ more accurate. \end{remark}

%\subsection{Logsig-RNN algorithm}\label{SectionAlgorithm}
The Logsig-RNN model (depicted in Figure~\ref{RNN_SDE}) can be served as an alternative to the RNN model and its variants of RNNs, e.g. LSTM, GRU. One main advantage of our method is to reduce the time dimension of the RNN model significantly while using the log-signature as an effective representation of data stream over sub-time interval. It leads to higher accuracy and efficiency compared with the standard RNN model. Compared with Sig-OLR (\cite{levin2013learning}) our approach achieves better accuracy via dimension reduction by using the log signature sequence of lower degree to represent the signature of high degree.

\subsection{Path Transformation Layers}\label{Ch_LP_Logsig_RNN}
To efficiently and effectively exploit the path, we further propose two transformation layers accompanied with Log-Signature Layer. The overall model, namely PT-Logsig-RNN, is shown in Figure \ref{LP_Logsig_RNN}.

\textbf{Embedding Layer} In many real-world applications, the input path dimension is large and the dimension of the truncated log-signature grows fast w.r.t. the path dimension. To reduce the high path dimension, we add a linear Embedding Layer before the Log-Signature Layer. The mapping $L$, implemented by the embedding layer, translates the input sequence $(X_{t_{i}})_{i = 1}^{n}$ into real vectors $(LX_{t_{i}})_{i = 1}^{n}$, where $LX_{t_{i}} \in \mathbb{R}^{d'}$ and $d' < d$. The weights in this layer are trainable and are learned from data. The embedding layer leads to significant spatial dimension reduction of rear layers.

\begin{figure}[t]
  \centering
  \includegraphics[width= 0.45\textwidth]{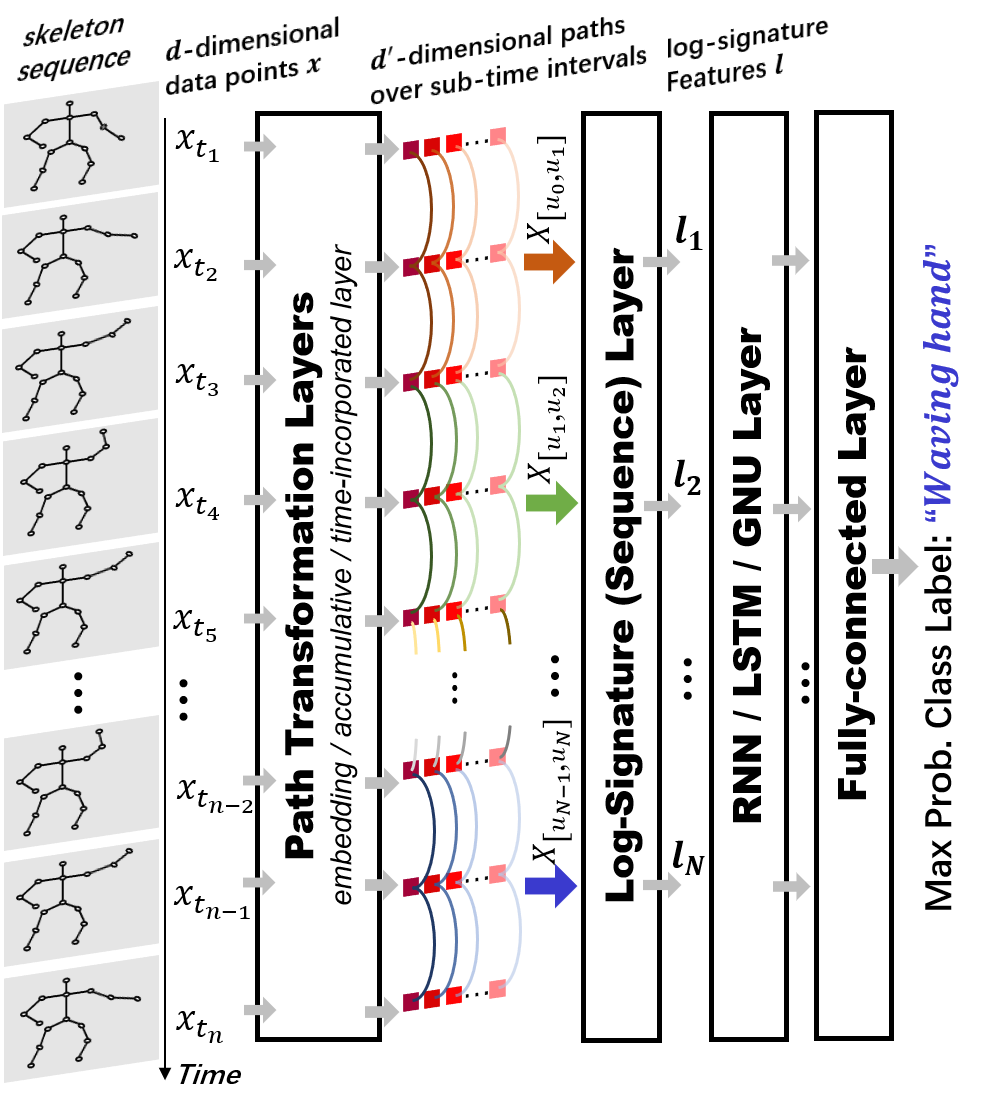}
  \caption{Architecture of PT-Logsig-RNN Model. It consists with the first Path Transformation Layers, the Log-Signature (Sequence) Layer, the RNN-type layer and the last fully connected layer. It is used for both action and gesture recognition in our experimental section.
}\label{LP_Logsig_RNN}
\end{figure}

\textbf{Accumulative Layer}
The Accumulative Layer maps the input sequence $(X_{t_{i}})_{i = 1}^{n}$ to its partial sum sequence $Y_{t_{i}}$, where $Y_{t_{i}} = \sum_{j = 1}^{i}X_{t_{j}}$, and $i=1, \cdots, n$. One advantage of using the Accumulative Layer along with Log-Signature Layer is to extract the quadratic variation and other higher order statistics of an input path $X$ effectively \cite{ni2015multi}.  

\textbf{Time-incorporated Layer}
The Time-incorporated Layer is to add the time dimension to the input sequence $(X_{t_{i}})_{i = 1}^{n}$; in formula, the output is $(t_{i}, X_{t_{i}})_{i = 1}^{n}$. The log-signature of the time-incorporated transformation of a path is proved to fully recover the path by Lemma \ref{Lemma_mono_logsig}. 

\section{Universality of Logsig-RNN Network}\label{sectionLearnSDEs}
%\subsection{Learning a controlled differential equation}
In this section, we prove the universality of Logsig-RNN network to approximate a solution to any controlled differential equation under mild conditions. The motivation of the Logsig-RNN network comes naturally from the numerical approximation theory of the SDEs.

Let $(X_t)_{t \in [0, T]}$ and $(Y_t)_{t \in [0, T]}$ be two stochastic processes under the probability space $(\Omega, \mathcal{F}, P)$ such that $Y$ is the solution to Equation \eqref{controlEquation} driven by the path $X$, 
\begin{eqnarray}
dY_{t} = f(Y_{t}) dX_{t}, Y_{S} = \xi, \label{controlEquation}
\end{eqnarray}
where $X$ has finite $p$-variation a.s., and $f: \mathbb{R} \rightarrow L(E, \mathbb{R})$ is a smooth vector field satisfying certain regularity condition. Let $I_{f}$ denote the solution map which maps $X_{J}$ to $Y_{T}$. The goal is to learn the solution map $I_{f}$ from the input-output pairs $(X_{J}^{(i)}, Y^{(i)}_{T})_{i = 1}^{N_s}$, which are iid samples of $(X_{J}, Y_{T})$.
%\subsection{Theoretical Background}\label{SectionTheory}
Classical numerical schemes of simulation of the solution to Equation \eqref{controlEquation} are mainly composed with two steps:
\begin{enumerate}
    \item Local Approximation of $Y_{t} - Y_{s}$ when $t$ is close to $s$. 
    \item Paste the local approximation together to get the global approximation for $Y_{T}$.
\end{enumerate}
Let us start with the local approximation of the solution to Equation \eqref{controlEquation} using step $M$-Taylor Expansion, i.e.
\begin{eqnarray}
%\begin{split}
Y_{t} - Y_{s} \approx\sum_{k = 1}^{M} f^{\circ k}(Y_{s}) \underset{s < s_{1} < \cdots < s_{k}<t}{\int}dX_{t_{1}}\otimes \cdots \otimes dX_{s_{k}},
%\end{split}
\end{eqnarray}
where $f^{\circ m}: \mathbb{R} \rightarrow L(E^{\otimes m}, \mathbb{R})$ is defined recursively by 
\begin{eqnarray}
f^{\circ 1} = f; f^{\circ k+1} = D(f^{\circ k}) f, 
\end{eqnarray}
where $D(g)$ denotes the differential of the function $g$.

We paste the local Taylor approximation together to estimate for the solution on the whole time interval $J$. The strategy is outlined as follows. Fix a time partition $\mathcal{D} = (u_{k})_{k =0}^{N}$ of $J$. We define the estimator $(\hat{Y}_{u_{k}}^{\mathcal{D}, M})_{k}$ given by $M$-step Taylor expansion associated with $\mathcal{D}$ in the following recursive way: for $k \in \{0, \cdots, N-1\}$,
\begin{eqnarray}\label{RecursiveStructureY}
{\hat{Y}_{u_{0}} ^{\mathcal{D}, M}}&=& y_{0} \nonumber,\\
{\hat{Y}_{u_{k+1}}^{\mathcal{D}, M }}&=& \hat{Y}_{u_{k}} ^{\mathcal{D}, M}+ \sum_{j = 1}^{M} f^{\circ j}(\hat{Y}_{u_{k}}^{\mathcal{D}, M}) X_{u_{k}, u_{k+1}}^{j}\\
&:=& g_{\mathcal{D}, M}^{f}( l_{k}, {\hat{Y}_{u_{k}}^{\mathcal{D}, M}} )\nonumber,
\end{eqnarray}
where $M$ is the degree of log-signature and $l_{k}$ is defined in Equation \eqref{def_ls}. $\hat{Y}_{u_{N}}^{\mathcal{D}, M }$ converges to $Y_{T}$ when $\Delta \mathcal{D}$ tends to $0$ provided $f$ is smooth enough (see Theorem \ref{GlobalApproxThm2}).

Equation \eqref{RecursiveStructureY} exploits a remarkable similarity between the recursive structure of an RNN $\mathcal{R}_{\sigma}$ and the one defined by the numerical Taylor approximation to solutions of SDEs $\hat{Y}_{u_{N}}$. It is noted that in Equation \eqref{RecursiveStructureY} $\hat{Y}_{u_{k+1}}^{\mathcal{D}, M}$ depends on $\hat{Y}_{u_{k}}^{\mathcal{D}, M}$ and the log-signature $l_{k}$. In this way, $\hat{Y}_{u_{k}}^{\mathcal{D}, M}$ plays a role similar to the hidden neurons of the RNN model with $(l_{k})_{k=0}^{N-1}$ as an input (see Figure \ref{SDE_sol_vs_RNN}). It motivates us to propose the Logsig-RNN Network (Model \ref{logsigrnn}) to approximate the solution to any SDE under the regularity condition. This idea is natural if one thinks that $g_{\mathcal{D}, M}^{f}$ can be universally approximated by neural network. We establish the universality theorem of the Logsig-RNN network as follows. The proof can be found in the appendix.

\begin{theorem}[Universality of Logsig-RNN Model]\label{UniversalityLogsigRNN} Let $Y$ denote the solution of a SDE of form (\ref{controlEquation}) under the previous regularity condition of Theorem \ref{GlobalApproxThm}. Let $K$ be any a compact set $S(V_{p}(J, E))$. Assume $f \in \mathcal{C}_{b}^{\infty}$\footnote{$f \in \mathcal{C}_{b}^{\infty}$ means $f$ is infinitely differentiable and all the derivatives are bounded. This regularity condition can be weaken.}. For any $\varepsilon >0$, there exist the constants $C_{1}:=C_{1}(p, \gamma, f, K)$ and $C_{2}:=C_{2}(f, K)$ such that $M > \lfloor p\rfloor$ and $\Delta \mathcal{D} \leq \min(\varepsilon^{p/ M +1 -p }C_{1},\varepsilon C_{2} )$, $l_{k}$ is defined in Equation \eqref{def_ls}. Then there exists a RNN $\mathcal{R}_{\sigma}(\vert \Theta)$ with some $\Theta$, s.t.
\begin{eqnarray}
\sup_{S(X) \in K}\vert\vert Y_{T} - \mathcal{R}_{\sigma}((l_{k})_{k = 1}^{N} \vert \Theta) \vert\vert \leq \varepsilon.
\end{eqnarray}
\end{theorem}

\section{Numerical Experiments}\label{SectionNumerics}
%and RNN$_{\mathcal{D}}$\footnote{RNN$_{\mathcal{D}}$ means the RNN model with $(s_{k})_{k = 1}^{\vert\mathcal{D}\vert}$ with $s_{k} = (x_{t_{i}})_{t_{i} \in [u_{k}, u_{k+1}]}$ as the input.}resp.
We demonstrate the performance of the Logsig-RNN algorithm on both synthetic SDE data and empirical data, including NTU RGB+D 120 action data and Chalearn 2013 gesture data in terms of the accuracy, efficiency and robustness.\footnote{We implement all algorithms in Tensorflow. It runs on a computer equipped with GeForce RTX 2080 Ti GPU.} 

\subsection{Synthetic Data Generated by a SDE}
As an example of high frequency data, we simulate the solution $Y_{T}$ to the SDE of Example~\ref{examp2} using Milstein's method with the time step $\frac{T}{50000}$ for $T = 10$. An input path is the discretized Brownian motion $W_{\hat{\mathcal{D}}}$, where $\hat{\mathcal{D}} = D_{50001}$\footnote{$D_{n}$ denotes an equally spaced  partition of $[0, T]$ of $n$ steps.}. We simulate $2000$ samples of $(X_{\hat{\mathcal{D}}}, Y_{T})$, which is split to $80\% $ for the training and the rest for the testing. Here we benchmark our approach with (1) RNN$_{0}$: the conventional RNN model, (2) Sig-OLR: the linear model on the signature, (3) Sig-RNN: the RNN model
 with the signature sequence. 
\begin{example}
Suppose $Y_t$ satisfies the following SDE:
\begin{eqnarray}
dY_t = (-\pi Y_t + \sin(\pi t))dX^{(1)}_t + Y_t dX^{(2)}_t, Y_{0} = 0,
\end{eqnarray}
where $X_t = (X^{(1)}_t, X^{(2)}_t) = (t, W_t)$, $W_{t}$ is a 1-d Brownian motion,  and the integral is in the Stratonovich sense.
\label{examp2}
\end{example}

\begin{table}[H]
\centering
\scalebox{0.85}{
\begin{tabular}{l|l|c|c|c}
\hline
\textbf{Data} & \textbf{Methods} & \textbf{Fea. dim.} & \textbf{Error($\times 10^{-6}$)} & \textbf{Train time(s)}\\
\hline
High & RNN$_0$ & (50k, 1) & $-$& $-$\\
Frequency& Sig-OLR & 62 & 2.25 & \textbf{178}\\
(50k steps)& Sig-RNN & (4,14) & 2.40 & 360 \\
& Logsig-RNN & \textbf{(4,8)} & \textbf{2.14} & 529\\
\hline
Down-& RNN$_0$ & (1k,1) & 7.79 & 50930\\
sampling& Sig-OLR & 62 & 3.69 & \textbf{9} \\
(1k steps)& Sig-RNN & (4,14) & 2.55 & 177 \\
& Logsig-RNN & \textbf{(4,8)} & \textbf{2.16} & 343\\
\hline
Missing & RNN$_0$ & (1k,1) & 16.40 & 47114\\
Data& Sig-OLR & 62 & 3.75 & \textbf{9} \\
(drop 5\% & Sig-RNN & (4,14) & 3.05 & 182 \\
from 1k)& Logsig-RNN & \textbf{(4,8)} & \textbf{2.91} & 372\\
\hline
\end{tabular}}
\caption{Comparison of methods on the SDEs data.}
\label{Tab:diffscenarios}
\end{table}

%\subsubsection{Accuracy and Efficiency Comparison}
As shown in Table \ref{Tab:diffscenarios}, we apply the above four methods for three kinds of inputs (1)  $X_{\hat{\mathcal{D}}}$ (high frequency); (2) down-sampling $X_{\hat{\mathcal{D}}}$ to $1k$ time steps (downsampling); (3) randomly throw away $5\%$ points of $1k$ down sampled data (missing data). We compare the accuracy and training time of the Logsig-RNN algorithm. The training time is the first time of the loss function of the model to reach the error tolerance level $2*10^{-6}$ before $25$k epochs in the train set and the MSE is chosen as performance metric. First of all, Table \ref{Tab:diffscenarios} shows that the Logsig-RNN achieves the best accuracy for all three cases among all the methods. In particular, it is the most robust to missing data. Moreover, it reduces the time dimension of RNN from $50k/1k$ to $4$, and thus significantly save the training time from $50930s$ to $343s$.

\subsection{Action Recognition: NTU RGB+D 120 Data}
NTU RGB+D 120 \cite{Liu_2019_NTURGBD120} is a large-scale benchmark dataset for 3D action recognition, which consists of $114,480$ RGB+D video samples that are captured from 106 distinct human subjects for 120 action classes. 

We use the network constructed by the path transformation layers following by the Logsig-RNN model (i.e. PT-Logsig-RNN shown in  Figure \ref{LP_Logsig_RNN}) and apply it to the skeleton data of the NTU RGB+D 120 data (150 dimensional data streams) for action classification. The network configurations are given in appendix.

\begin{table}[H]
\centering
\scalebox{0.75}{
\begin{tabular}{l|c|c}
\hline
\textbf{Methods}                     & \textbf{X-Subject(\%)}&
\textbf{X-Setup(\%)}\\
\hline
Dynamic Skeleton\cite{7784788} & 50.8 & 54.7\\
ST LSTM\cite{10.1007/978-3-319-46487-9_50} & 55.7 & 57.9\\
FSNet\cite{liufsnet} & 59.9& 62.4\\
TS Attention LSTM\cite{Liu2018SkeletonBasedHA} & 61.2 & 63.3\\
MT-CNN + RotClips\cite{8306456} & 62.2 & 61.8\\
Pose Evolution Map\cite{liu2018recognizing} & 64.6 & 66.9\\
\hline
LSTM (baseline) & 61.6& 58.5\\
PT + LSTM & 62.0& 60.5\\
PT + Logsig +LSTM & \textbf{65.7}& 64.5\\
\hline
\end{tabular}}
\caption{Comparison of methods on the NTU RGB+D 120.}
\label{Accuracy_Action_NTU120}
\end{table}

As shown in Table \ref{Accuracy_Action_NTU120}, we subsequently add the Path Transformation Layers (PT) and the Log-signature layer (Logsig) to the baseline LSTM to validate the performance of each model. For X-Subject task, adding PT Layer results in a 0.4 percentage points (\textit{pp}) gain over the baseline and the Logsig layer further gives a 3.7 \textit{pp} gain. For X-Subject protocol, our method outperforms other methods. For X-Setup, our method is only beaten by ~\cite{liu2018recognizing}. The latter leverages the informative pose estimation maps as additional clues. Notice that our PT-Logsig-LSTM is flexible enough to allow incorporating other advanced techniques (e.g. data augmentation and attention module) or combining multimodal clues (e.g. pose confidence score) to achieve further improvement.

\subsection{Gesture Recognition: Chalearn 2013 data}\label{chalearn2013}
The Chalearn 2013 dataset \cite{Escalera2013MultimodalGR} is a public available dataset for gesture recognition, which contains 20 Italian gestures performed by 27 subjects. It provides Kinect data, which contains RGB, depth, foreground segmentation and skeletons. Here, we only use skeleton data (20 3D joints) for the gesture recognition (see appendix for more details). 

We compare our method (Figure \ref{LP_Logsig_RNN}) with several state-of-the-art methods \cite{li2019skeleton}. Table \ref{tableofchalearn} shows that the PT-Logsig-RNN algorithm with $M=2$ and $N=4$ outperforms other methods in terms of the accuracy. We present both the results with/without the data-augmentation. With augmentation, our results significantly outperform others to achieve the state-of-the-art result.

\begin{table}[H]
\centering
\scalebox{0.75}{
\begin{tabular}{l|c|c}
\hline
\textbf{Methods}                     & \textbf{Accuracy(\%)}   & \textbf{Data Aug.}   \\ \hline
Deep LSTM \cite{nturgb}               & 87.10   & $-$      \\
Two-stream LSTM \cite{Wang2017ModelingTD}           & 91.70      & $\surd$     \\
ST-LSTM + Trust Gate \cite{8101019}& 92.00 & $\surd$
   \\
3s\_net\_TTM \cite{li2019skeleton}              & 92.08  &  $\surd$      \\ \hline
RNN$_0$                   & 90.92    & $\bigtimes$       \\
RNN$_0$ (+data augmentation)                    &   91.18  & $\surd$       \\
\hline
\textbf{PT-Logsig-RNN} & \textbf{92.21} & $\bigtimes$ \\ 
\textbf{PT-Logsig-RNN}(+data augmentation)  & \textbf{93.27} & $\surd$ \\ \hline
\end{tabular}}
\caption{Comparison of methods on the Chalearn 2013 data.}
\label{tableofchalearn}
\end{table}

Regarding to the robustness to missing data, we randomly set a certain percentage of frames ($r$) by all-zeros for each sample in the validation set, and evaluate the trained models of our method and RNN$_0$ to the new validation data. Table~\ref{gesturemissingdata} shows that logsig-RNN model with $M=2$ consistently beats the baseline $\text{RNN}$ for different $r$, which validates the robustness of our method comparing with the benchmark.

\begin{table}[H]
  \centering
   \scalebox{0.90}{ \begin{tabular}{|c|c|c|c|c|}
   \hline
    \diagbox[width=5em]{$r$}{$M$} & 2     & 3     & 4     & RNN$_0$\\
    \hline
    0\%   & \cellcolor[rgb]{ .439,  .678,  .278}92.21 & \cellcolor[rgb]{ .663,  .816,  .557}89.66 & \cellcolor[rgb]{ 1,  .902,  .6}70.71 & \cellcolor[rgb]{ .573,  .816,  .314}90.92 \\
    \hline
    10\%  & \cellcolor[rgb]{ .573,  .816,  .314}91.32 & \cellcolor[rgb]{ .663,  .816,  .557}88.58 & \cellcolor[rgb]{ 1,  .851,  .4}69.11 & \cellcolor[rgb]{ .776,  .878,  .706}81.77 \\
    \hline
    20\%  & \cellcolor[rgb]{ .663,  .816,  .557}90.33 & \cellcolor[rgb]{ .776,  .878,  .706}86.60 & \cellcolor[rgb]{ 1,  .851,  .4}67.89 & \cellcolor[rgb]{ 1,  .851,  .4}68.22 \\
    \hline
    30\%  & \cellcolor[rgb]{ .776,  .878,  .706}87.68 & \cellcolor[rgb]{ .776,  .878,  .706}81.74\% & \cellcolor[rgb]{ 1,  .745,  .051}63.24 & \cellcolor[rgb]{ 1,  .392,  .42}50.35 \\
    \hline
    50\%  & \cellcolor[rgb]{ 1,  .902,  .6}74.40 & \cellcolor[rgb]{ 1,  .745,  .051}57.13 & \cellcolor[rgb]{ 1,  .392,  .42}41.07 & \cellcolor[rgb]{ 1,  .027,  0}21.78 \\
    \hline
    \end{tabular}}%
    \caption{The accuracy (\%) of the testing set with missing data with different dropping ratio ($r$). Here $N=4$.}
  \label{gesturemissingdata}%
\end{table}%

\section{Conclusion}\label{SectionFutureWork}
The Logsig-RNN model, inspired from the numerical approximation theory of SDEs, provides an accurate, efficient and robust algorithm to learn a functional on streamed data. Numerical results show that it improves the performance of LSTM significantly on both synthetic data and empirical data. In ChaLearn2013 gesture data, PT-Logsig-RNN achieves the state-of-the-art classification accuracy. It is noted that the gesture or action data is naturally one kind of enormous continuous data streams in real world. When devices make higher frequency sampled data available, the proposed algorithm can be very suited in related tasks, while conventional downsampling-based RNNs probably fail.  

\include{appendix}
%\section*{Acknowledgements}
%HN and TL are supported by the EPSRC under the program grant EP/S026347/1 and by the Alan Turing Institute under the EPSRC grant EP/N510129/1. WY is supported by Royal Society Newton International Fellowship.
%\end{acknowledgement}
\bibliographystyle{ieee}
%\bibliography{C:/Users/NiHao/Dropbox/myref}{}
\bibliography{myref}{}

\end{document}

%% file: appendix.tex
\begin{appendices}\label{appendix}
%\appendix\label{appendix}
\begin{center}
    \huge \textbf{Appendix}
\end{center}

\section{Preliminary of Rough Paths Theory}
In this section, we give a brief overview of the signature and log signature of a path, and provide the necessary preliminary of Rough Path Theory. Besides we provide the pendigit$\_$demo.ipynb\footnote{ pendigit$\_$demo.ipynb is provided in the gitub via \url{https://github.com/logsigRNN/learn_sde/blob/master/Pen-digit_learning/pendigit_demo.ipynb}} as illustrative examples to help readers have a better understanding of the properties of the (log)-signature. 

\subsection{The signature of a path}
Let us recall the definition of the signature of a path.
\begin{definition}[The Signature of a Path]
Let $J$ denote a compact interval and $X : J \rightarrow E$ be a continuous path with finite $p$-variation such that the following integration makes sense. Let $I=(i_{1}, i_{2}, \cdots, i_{n})$ be a multi-index of length $n$ where $i_{j} \in \{1, \cdots, d\}, \forall j \in \{1, \cdots, n\}$. Define the coordinate signature of the path $X_{J}$ associate with the index $I$ as follows:
\begin{equation*}
    X^{I}_{J} = \underset{\underset{u_{1}, \dots, u_{k} \in J}{u_{1} < \dots < u_{k}}} { \int \dots \int} dX_{u_{1}}^{(i_{1})} \otimes \dots \otimes dX_{u_{n}}^{(i_{n})}
\end{equation*}
The signature of $X$ is defined as follows:
\begin{equation}
    S(X)_J = (1, \mathbf{X}_J^1, \dots, \mathbf{X}_J^k, \dots)
\end{equation}
where $\displaystyle \mathbf{X}_J^k = \underset{\underset{u_{1}, \dots, u_{k} \in J}{u_{1} < \dots < u_{k}}} { \int \dots \int} dX_{u_{1}} \otimes \dots \otimes dX_{u_{k}} = (X^{I}_{J})_{I = (i_{1}, \cdots, i_{k})}, \forall k \geq 1$. 

Let $S_{k}(X)_{J}$ denote the truncated signature of $X$ of degree $k$, i.e.
\begin{equation}
    S_{k}(X)_J = (1, \mathbf{X}_J^1, \dots, \mathbf{X}_J^k).
\end{equation}
\end{definition}

The signature of the path has geometric interpolation. The first level signature $\mathbf{X}_J^1$ is the increment of the path $X$, i.e $X_{T} - X_{S}$, while the second level signature represents the signed area enclosed by the curve $X$ and the cord connecting the ending and starting point of the path $X$.
 
The signature of $X$ arises naturally as the basis function to represent the solution to linear controlled differential equation based on the Picard's iteration (\cite{RoughPaths}). It plays the role of non-commutative monomials on the path space. In particular, if $X$ is a one dimensional path, the $k^{th}$ level of the signature of $X$ can be computed explicitly by induction as follows that for every $k \in \mathbb{N}$,
\begin{eqnarray}
 \mathbf{X}_J^k = \frac{(X_{T} - X_{S})^{k}}{k!}.
\end{eqnarray}
\subsubsection{Multiplicative Property}
The signature of paths of finite $1-$variation has the multiplicative property, also called Chen's identity. 
\begin{definition}
Let $X: [0, s] \rightarrow E$ and $Y: [s, t] \rightarrow E$ be two continuous paths. Their concatenation is the path denoted by $X * Y: [0, t] \rightarrow E$ defined by
\begin{eqnarray*}
(X*Y)_{u} = \begin{cases} X_{u}, &  u \in [0,s], \\ Y_{u}-Y_{s}+X_{s}, & u \in [s, t]. \end{cases}
\end{eqnarray*}
\end{definition}
\begin{theorem}[Chen's identity]\label{Chen}. Let  $X: [0, s] \rightarrow E$ and $Y: [s, t] \rightarrow E$ be two continuous paths of bounded one-variation. Then
\begin{eqnarray*}
S(X*Y) = S(X) \otimes S(Y).
\end{eqnarray*}
\end{theorem}
Chen's identity asserts that the signature is a homomorphism between the path space and the signature space.
\subsubsection{Calculation of the signature}
In this subsection, we explain how to compute the truncated signature of a piecewise linear path.
Let us start with a $d$-dimensional linear path.
\begin{lemma}
Let $X: [S, T] \rightarrow E$ be a linear path. Then 
\begin{equation}
    S^{n}(X) = \frac{(X_{T} - X_{S})^{\otimes n}}{n!}.
\end{equation}
Equivalently speaking, for any multi-index $I = (i_{1}, \cdots, i_{n})$,
\begin{equation}
    S^{I} = \frac{\prod_{j=1}^{n}(X_{T}^{(i_{j})})}{n!}
\end{equation}
\end{lemma}
Chen's identity is a useful tool to enable compute the signature of the piecewise linear path numerically. 
\begin{lemma}
Let $X$ be  a $E$-valued  piecewise  linear  path,  i.e.$X$is  the concatenation  of  a  finite  number  of  linear  paths,  and  in  other  words  there exists a positive integer $l$and linear paths$X_1,X_2,...,X_l$such that $X=X_1*X_2*···*X_l$. Then
\begin{eqnarray}
S(X) = \otimes_{i = 1}^{l} \exp(X_{i}).
\end{eqnarray}
\end{lemma}

\subsubsection{Uniqueness of the signature}
Let us start with introducing the definition of the tree-like path.
\begin{definition}[Tree-like Path]\label{def_tree_like}
A path $X: J=[S, T] \rightarrow E$ is tree-like if there exists a continuous function $h: J \rightarrow [0, +\infty)$ such that $h(S) = h(T) = 0$ and such that, for all $s, t \in J$ with $s \leq t$, 
\begin{eqnarray*}
\vert\vert X_{t} - X_{s}\vert\vert \leq h(s) + h(t) - 2 \inf_{u \in [s, t]} h(u).
\end{eqnarray*}
\end{definition}
Intuitively a tree-like path is a trajectory in which a section where the path exactly retraces itself. The tree-like equivalence is defined as follows: we say that two paths $X$ and $Y$ are the same up to the tree-like equivalence if and only if the concatenation of $X$ and the inverse of $Y$ is tree-like. Now we are ready to characterize the kernel of the signature transformation.
\begin{theorem}[Uniqueness of the signature]\label{UniquenessOfSig}
	Let $X \in V_{p}(J,E)$ . Then $S(X)$ determines $X$ up to the tree-like equivalence defined in Definition \ref{def_tree_like}.\cite{UniquenessOfSignature}  %\cite{boedihardjo2014signature},   
\end{theorem}
Theorem \ref{UniquenessOfSig} shows that the signature of the path can recover the path trajectory under a mild condition. The uniqueness of the signature is important, as it ensures itself to be a discriminative feature set of un-parameterized streamed data.
\begin{remark}
	A simple sufficient condition for the uniqueness of the signature of a path of finite length is that one component of $X$ is monotone. Thus the signature of the time-joint path determines its trajectory (see \cite{levin2013learning}). 
\end{remark}

\subsubsection{Invariance under time parameterization}
\begin{lemma}[Invariance under time parameterization]\label{SigTimeInvariance}\cite{RoughPaths}
	Let $X \in V_{1}(J, E)$  and a path $\tilde{X}: J \rightarrow E$ is the time re-parameterization of $X$. Then 
	\begin{eqnarray}\label{sig_time_inv}
	S(X)_{J} = S(\tilde{X})_{J}.
	\end{eqnarray} 
\end{lemma}
Re-parameterizing a path inside the interval does not change its signature. In Figure \ref{TimeParameterizationInvariance}, speed changes result in different time series representation but the same signature feature. It means that signature feature can reduce dimension massively by removing the redundancy caused by the speed of traversing the path. It is very useful for the applications where the output is invariant w.r.t. the speed of an input path, e.g. online handwritten character recognition and video classification. 
\subsubsection{Shuffle Product Property} We introduce a special class of linear forms on $T((E))$; Suppose $(e^{*}_{1}, \cdots, e^{*}_d, \cdots) $ are elements of $E^{*}$. We can introduce coordinate iterated integrals by setting
$X^{(i)}_{u}:= \langle e^{*}_{i}, X_{u} \rangle $,
and rewriting 
$\langle e_{i1}^{*} \otimes \cdots \otimes e_{in}^{*}, S(X) \rangle$ as the scalar iterated integral of coordinate projection. In this way, we realize $n^{th}$ degree coordinate iterated integrals as the restrictions of linear functionals in $E^{\otimes n}$to the space of signatures of paths. If $(e_1,\cdots,e_d )$ is a basis for a finite dimensional space $E$, and $(e_1^{*}, \cdots ,e_d^{*})$ is a basis for the dual $E^{*}$. Therefore, it follows that
\begin{eqnarray*}
\mathbf{X}_{J} = \sum_{ \substack{k \geq 0\\i_{1}, \cdots, i_{k}  \\ \{\in 1, 2, \cdots, d\}}}\underset{\underset{u_{1}, \dots, u_{k} \in J}{u_{1} < \dots < u_{k}}} { \int \dots \int} dX_{u_{1}}^{(i_{1})} \otimes \dots \otimes dX_{u_{k}}^{(i_{k})} e_{1}\otimes \cdots \otimes e_{k}.
\end{eqnarray*}

\begin{theorem}[Shuffle Algebra]
The linear forms on $T((E))$ induced by $T(E^{*})$, when restricted to the range $S(\mathcal{V}^{p}([0, T], E)$ of the signature, form an algebra of real valued functions for $p<2$.
\end{theorem}
The proof can be found in page 35 in \cite{RoughPaths}. The proof is based on the Fubini theorem, and it is to show that for any $e^{*}, f^{*} \in T(E^{*})$, such that for all $\mathbf{a} \in S(\mathcal{V}^{p}([0, T], E)$,
\begin{equation}
    e^{*}(\mathbf{a}) f^{*}(\mathbf{a}) = (e^{*} \shuffle f^{*})(\mathbf{a})
\end{equation}
\subsubsection{Universality of the signature}
Any functional on the path can be rewritten as a function on the signature based on the uniqueness of the signature (Theorem \ref{UniquenessOfSig}). The signature of the path has the universality, i.e. that any continuous functional on the signature can be well approximated by the linear functional on the signature (Theorem \ref{SigApproximatinTheorem})\cite{levin2013learning}.

\begin{theorem}[Signature Approximation Theorem]\label{SigApproximatinTheorem}
	Suppose $f: S_{1} \rightarrow \mathbb{R}$ is a continuous function, where $S_{1}$ is a compact subset of $S(V_{p}(J,E))$\footnote{ $S(V_{p}(J,E))$ denotes the range of the signature of $x \in V_{p}(J,E)$.}. Then $\forall \varepsilon >0$, there exists a linear functional $L \in T((E))^{*}$ such that
	\begin{eqnarray}
	\sup_{a \in S_{1}}\vert \vert f(a) - L(a) \vert \vert \leq \varepsilon. 
	\end{eqnarray}
\end{theorem}
\begin{proof}
It can be proved by the shuffle product property of the signature and the Stone-Weierstrass Theorem. 
\end{proof}

\subsection{The log-signature of a path}
\subsubsection{Lie algebra and Lie series}
If $F_{1}$ and $F_{2}$ are two linear subspaces of $T((E))$, let us denote by $[F_{1}, F_{2}]$ the linear span of all the elements of the form $[a,b]$, where $a \in F_{1}$ and $b \in F_{2}$.
Consider the sequence $(L_{n})_{n \geq 0}$ be the subspace of $T((E))$ defined recursively as follows:
\begin{equation}
    L_{0} = 0;  \forall n \geq 1,  L_{n} = [E, L_{n-1}]. 
\end{equation}
\begin{definition}\label{def_lie_seires}
The space of Lie formal series over $E$, denoted as $\mathcal{L}((E))$ is defined as the following subspace of $T((E))$:
\begin{eqnarray}
\mathcal{L}((E)) = \{l = (l_{0}, \cdots, l_{n}, \cdots) \vert \forall n \geq 0, l_{n} \in L_{n}\}.
\end{eqnarray}
\end{definition}
Theorem \ref{theorem_log_sig_Lie_Series} can be rewritten in the following form.
\begin{theorem}[Theorem 2.23 \cite{RoughPaths}]
Let $X$ be a path of finite $1$-variation. Then the log-signature of $X$ is a Lie series in $\mathcal{L}((E))$.
\end{theorem}

\subsubsection{The bijection between the signature and log-signature}
Similar to the way of defining the logarithm of a tensor series, we have the exponential mapping of the element in $T((E))$ defined in a power series form.
\begin{definition}[Exponential map]
	Let $a = (a_{0}, a_{1}, \cdots) \in T((E))$. Define the exponential map denoted by $\exp$ as follows:
	\begin{eqnarray}
	\exp (a) = \sum_{n = 0}^{\infty}\frac{a^{\otimes n}}{n!}.
	\end{eqnarray}
\end{definition}
\begin{lemma}
The inverse of the logarithm on the domain $\{a \in T((E)) | a_{0} \neq 0\}$ is the exponential map.
\end{lemma}

\begin{theorem}\label{dim_hall}
The dimension of the space of the truncated log signature of $d$-dimensional path up to degree $n$ over $d$ letters is given by:

$$\mathcal{DL}_n = \frac{1}{n} \sum_{d|n}\mu(d)q^{n|d}$$

where $\mu$ is the Mobius function, which maps $n$ to
\begin{eqnarray*}
 \left\{
    \begin{array}{ll}
        0, & \mbox{if $n$ has one or more repeated prime factors} \\
        1, & \mbox{if $n=1$} \\
        (-1)^k & \mbox{if $n$ is the product of k distinct prime numbers}
    \end{array}
\right.
\end{eqnarray*}
\end{theorem}
The proof can be found in  Corollary 4.14 p. 96 of \cite{reutenauer2003free}.

\subsubsection{Calculation of the log-signature}
Let's start with a linear path. The log signature of a linear path $X_{J}$ is nothing else, but the increment of the path $X_{T} - X_{S}$. \\

Baker-Cambpell-Hausdorff formula gives a general method to compute the log-signature of the concatenation of two paths, which uses the multiplicativity of the signature and the free Lie algebra. It provides a way to compute the log-signature of the piecewise linear path by induction. 
\begin{theorem}
For any $S_1, S_2 \in \mathcal{L}((E))$
\begin{eqnarray}
Z = log(e^{S_1}e^{S_2})=  \sum_{\substack{n\geq 1 \\ p_1,...,p_n \geq 0\\q_1,....q_n \geq 0 \\ p_i+q_i >0}}\frac{(-1)^{n+1}}{n} \frac{1}{p_1!q_1!...p_n!q_n!}r(S_1^{p_1}S_2^{q_1}...S_1^{p_n}S_2^{q_n})
\end{eqnarray}
where $r:A^* \rightarrow A^*$ is the right-Lie-bracketing operator, such that for any word $w=a_1...a_n$ 
$$r(w)=[a_1,...,[a_{n-1}, a_n]...].$$
This version of BCH is sometimes called the Dynkin's formula. 
\end{theorem}

\begin{proof}
See remark of appendix 3.5.4 p. 81 in \cite{reutenauer2003free}.
\end{proof}
\subsubsection{Uniqueness of the log-signature}
Like the signature, the log-signature has the uniqueness stated in the following theorem.
\begin{theorem}[Uniqueness of the log-signature]\label{UniquenessOflogSig}
	Let $X \in V_{p}(J,E)$ . Then $lS(X)$ determines $X$ up to the tree-like equivalence defined in Definition \ref{def_tree_like}. 
\end{theorem}
Theorem \ref{UniquenessOflogSig} shows that the signature of the path can recover the path trajectory under a mild condition. 
\begin{lemma}\label{Lemma_mono_logsig}
A simple sufficient condition for the uniqueness of the log-signature of a path of finite length is that one component of $X$ is monotone. 
\end{lemma}

\subsection{Comparison of the Signature and Log-signature}
Both the signature and log-signature take the functional view on the discrete time series data, which allows a unified way to treat time series of variable length and missing data. For example, we chose one pen-digit data of length 53 and simulate 100 samples of modified pen trajectories by dropping at most 16 points from it, to mimic the missing data of variable length case (See one sample in Figure \ref{missing_data9}). Figure \ref{sig_logsig_comparison_missing_data} shows that the mean absolute relative error of the signature and log-signature of the missing data is no more than $6\%$. Besides the log-signature feature is more robust to missing data and of lower dimension compared with the signature feature.

\begin{figure}[!ht]
	\centering
	\includegraphics[width= 0.7\textwidth]{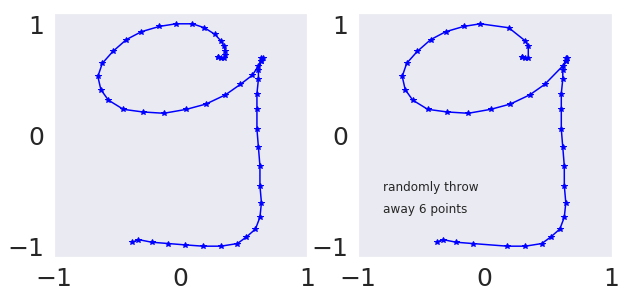}
	\caption{(Left) The chosen pen trajectory of digit 9. (Right) The simulated path by randomly dropping at most 16 points of the pen trajectory on the left.
	}\label{missing_data9}
\end{figure}
\begin{figure}[!ht]
	\centering
	\includegraphics[width= 0.7\textwidth]{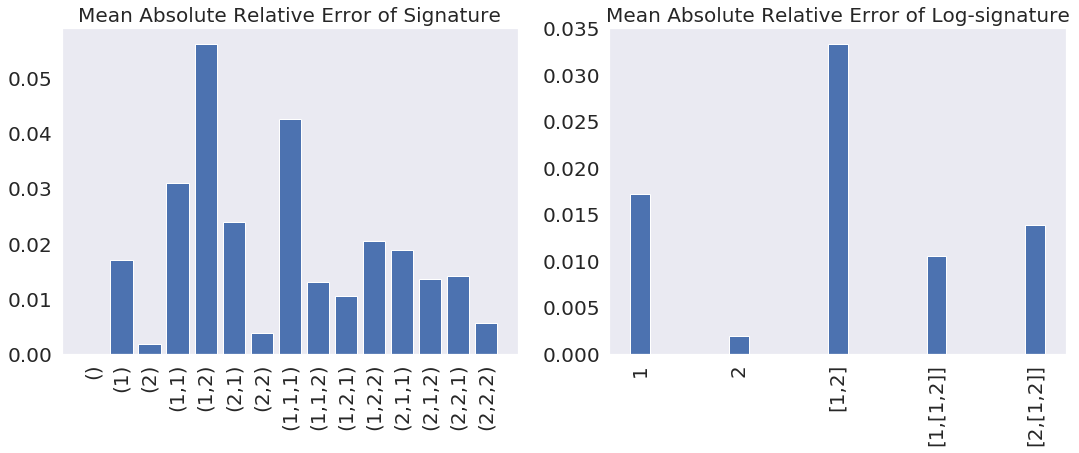}
	\caption{Signature and Log-Signature Comparison for the missing data case.
	}\label{sig_logsig_comparison_missing_data}
\end{figure}
\begin{figure}[!ht]
	\centering
	\includegraphics[width= 0.7\textwidth]{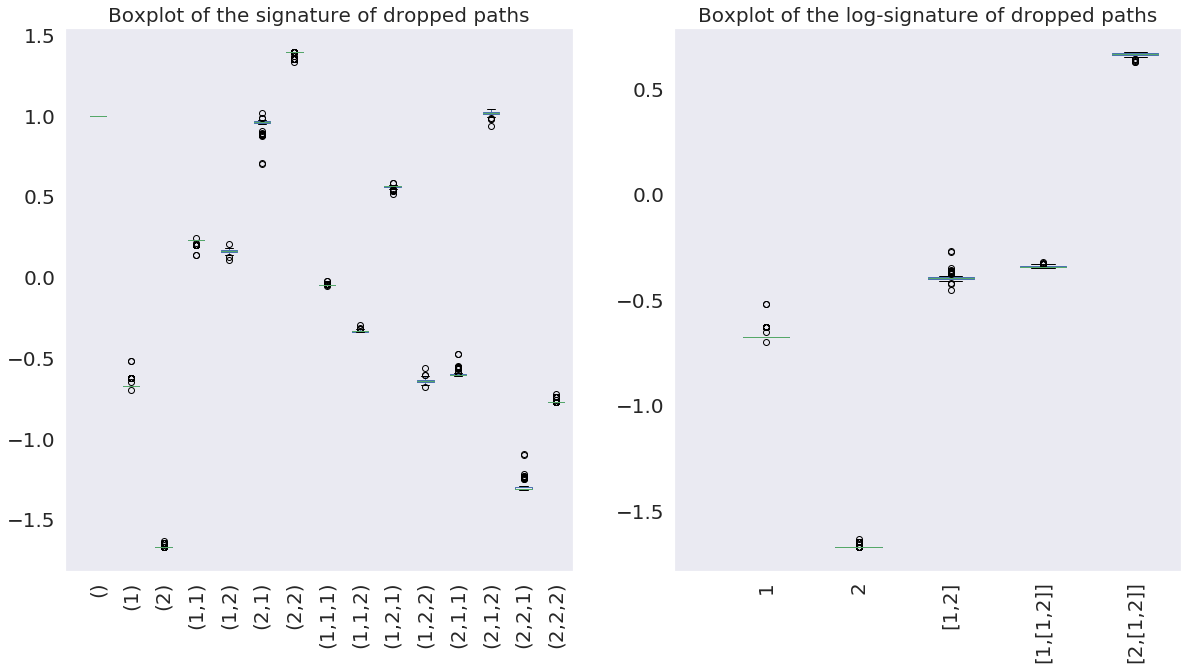}
	\caption{Signature and Log-Signature Comparison for the missing data case.
	}\label{boxplot_Sig_Logsig}
\end{figure}

\subsection{Rough Path and the Extension Theory}
Let $C_{0, p}(\Delta_{T}, T^{\lfloor p \rfloor}(E))$ be the space of all continuous functions from the simplex $\Delta_{T}:=\{(s, t)\vert 0 \leq s \leq t \leq T\}$ into the truncated tensor algebra $T^{\lfloor p \rfloor}(E)$. We define the $p$-variation metric on this linear space as follows: for $X, Y \in C_{0, p}(\Delta_{T}, T^{\lfloor p \rfloor}(E))$, set
\begin{eqnarray}
d_{p}(X, Y) = \max_{1\leq i \leq \lfloor p \rfloor} \left( \sup_{\mathcal{D} \subset [0, T]} \sum_{\mathcal{D}} \vert\vert X_{t_{l-1}, t_{l}}^{i} -  Y_{t_{l-1}, t_{l}}^{i}\vert\vert^{\frac{p}{i}}\right)^{\frac{1}{p}}. 
\end{eqnarray}

\begin{definition}
Let $p \geq 1$ be a real number and $n \geq 1$ be an integer. Let $\omega: [0, T] \rightarrow [0, \infty]$ be a control. Let $X: \Delta_{T} \rightarrow T^{(n)}$ be a multiplicative functional. We say that $X$ has finite $p$-variation on $\Delta_{T}$ controlled by $\omega$ if
\begin{eqnarray}
\vert\vert X_{s,t}^{i} \vert\vert \leq \frac{\omega(s, t)^{\frac{i}{p}}}{\beta (\frac{i}{p})!}, \forall (s, t) \in \Delta_{T}.
\end{eqnarray}
In general, we say that $X$ has finite $p$-variation if there exists a control such that the conditions above are satisfied. 
\end{definition}
The concept of the rough path theory is a generalization of the signature of a path of finite 1-variation. 
\begin{definition}[Rough path]
Let $E$ be a Banach space. Let $p \geq 1$ be a real number. A $p$-rough path in $E$ is a multiplicative functional of degree $\lfloor p \rfloor$ in $E$ with finite $p$-variation. The space of $p$-rough paths is denoted by $\Omega_{p}(V)$.
\end{definition}

\begin{definition}[Geometric rough path]
A geometric $p$-rough path is a $p$-rough path that can be expressed as a limit of 1-rough path in the $p$-variation distance defined above. The space of geometric $p$-rough path in E is denoted by $G\Omega_{p}(E)$.
\end{definition}

\begin{theorem}[Extension theorem]
Let $X$ and $Y$ be two multiplicative functional in $T^{(n)}(V)$ of finite $p$-variation, $n \geq \lfloor p \rfloor$ controlled by $\omega$. Suppose that for some $\varepsilon \in (0, 1)$,
\begin{eqnarray}\label{control1}
\vert\vert  X_{s, t}^{i} - Y_{s, t}^{i}\vert\vert \leq \varepsilon \frac{\omega(s, t)^{\frac{i}{p}}}{\beta \left( \frac{i}{p}\right)!},
\end{eqnarray}
for $i = 1, \cdots, n$ and for all $ (s, t) \in \Delta_{T}$. Then provided $\beta$ is chosen such that
\begin{eqnarray}
\beta \geq 2p^{2} \left(1+ \sum_{r = 3}^{\infty} \left( \frac{2}{r-2}\right)^{\frac{\lfloor p \rfloor }{p}}\right).
\end{eqnarray}
\end{theorem}

\section{Proof of Global Approximation Theorem}\label{Appendix_Proof_Theorem41}
In this section, we prove that the global approximation theorem of the high order Taylor expansion of the solution to the controlled differential equations.
\begin{theorem}[Global Approximation Theorem]\label{GlobalApproxThm2} Let $\hat{Y}_{u_{N}}^{\mathcal{D}, M}$ be defined as previously. For any $\varepsilon >0$, when $\Delta \mathcal{D} \leq \left(\frac{\varepsilon}{\tilde{C}}\right)^{p/(\lfloor \gamma \rfloor +1 -p)}$ and $M \geq \lfloor \gamma \rfloor$, then $\hat{Y}_{u_{N}}^{\mathcal{D}, M}$ satisfies that
\begin{eqnarray}
\vert\vert Y_{T} - \hat{Y}_{u_{N}}^{\mathcal{D}, M}\vert\vert \leq \varepsilon,
\end{eqnarray}
where $S(X)\in K$, $K$ is a compact set $S(V_{p}(J, E))$ and $\tilde{C}$ is a constant depending $p, \gamma$, the norm of $f$ and the radius of $K$ defined in Equation \eqref{C_Def}. 
\end{theorem}
Before proceeding to the proof of the above theorem, we need the following auxiliary lemma and classical results on the numerical approximation of the SDEs (Theorem \ref{GlobalApproxThm}). 

\begin{lemma}Let $K$ be a compact set of $G\Omega_{p}(J, E)$ for some $p \geq 1$. $J'$ is a compact sub-time interval of $J$. Then the mapping $F: K \rightarrow T^{\lfloor p \rfloor}(E)$, i.e. 
$x \mapsto \log(x_{J'})$ is continuous. The image of $K$ under the function $F$ is compact.
\end{lemma}

\begin{theorem}\label{GlobalApproxThm}\cite{friz2010multidimensional}
Assume that $X = (1, X^{1}, \dots, X^{\lfloor p \rfloor})$ is a $p$-geometric rough path\footnote{A geometric $p$-rough path is the limit of the sequence of the signature of paths of finite $1$-variation in the $p$-variation distance. A discrete time series of finite length is an example of geometric $p$-rough path $\forall p\geq 1$.}. Let $f$ be a $Lip(\gamma)$ vector field where $\gamma>p$. Then there exists $C:= C(p,\gamma)$ such that
\begin{eqnarray}\label{error_eqn}
\vert\vert Y_{T} - \hat{Y}_{u_{N}}^{\mathcal{D}, M}\vert\vert \leq C \sum_{k=1}^N \vert f\vert_{\circ \gamma}^{\lfloor \gamma \rfloor +1} \vert\vert X \vert\vert_{p-var;[t_{k-1},t_k]}^{\lfloor \gamma \rfloor +1}.
\end{eqnarray}
\end{theorem}
Now we are ready to prove Theorem \ref{GlobalApproxThm2}.
\begin{proof}
According to Theorem \ref{GlobalApproxThm}, there exists $C:= C(p,\gamma)$ such that
\begin{eqnarray}\label{error_eqn}
\vert\vert Y_{T} - \hat{Y}_{u_{N}}^{\mathcal{D}, M}\vert\vert \leq C \sum_{k=1}^N \vert f\vert_{\circ \gamma}^{\lfloor \gamma \rfloor +1} \vert\vert X \vert\vert_{p-var;[t_{k-1},t_k]}^{\lfloor \gamma \rfloor +1}.
\end{eqnarray}
It implies that the estimation error is of order $\Delta^{ \frac{\lfloor \gamma \rfloor +1}{p}}$.
As $S(X) \in K$, then it exists a constant $C_{1}>0$ s.t.
\begin{eqnarray}
\sup_{S(X) \in K}\vert \vert X\vert\vert_{p-var, J} \leq C_{1}.
\end{eqnarray}
Equation \eqref{error_eqn} implies that 
\begin{eqnarray}
\vert\vert Y_{T} - \hat{Y}_{u_{N}}^{\mathcal{D}, M}\vert\vert \leq \tilde{C} \Delta \mathcal{D}^{\frac{\lfloor \gamma \rfloor +1}{p} -1}. 
\end{eqnarray}
where $C:= C(p, \gamma)$ is given in Theorem \ref{GlobalApproxThm} and 
\begin{eqnarray}\label{C_Def}
\tilde{C} = C \max_{k =1}^{N} \left (\vert \vert f \vert \vert_{\circ \gamma}^{\vert \gamma \vert + 1}\right) C_{1}^{p}.
\end{eqnarray}
\end{proof}
\section{Proof of Universality of the Logsig-RNN model}\label{appendix_Lemma}
In this section, we prove that the universality of the Logsig-RNN model (Theorem \ref{UniversalityLogsigRNN} of our paper). Firstly, we introduce the auxiliary lemmas and then complete the proof of Theorem \ref{UniversalityLogsigRNN}.
\subsection{Auxiliary Lemmas }
In the following, we use the uniform norm of a function $\tilde{f}: K \rightarrow \mathbb{R}^{d}$, i.e.
\begin{eqnarray*}
\vert\vert \tilde{f} \vert\vert_{\infty, K}:=\sup_{x \in K}  \vert \tilde{f}(x) \vert.
\end{eqnarray*}
The following lemma on the universality of shallow neural network was proved by Funahshi (1989).
\begin{lemma}\label{lemma1}
Let $\sigma(x)$ be a sigmoid function (i.e. a non-constant, increasing, and bounded continuous function on $\mathbb{R}$). Let $K$ be any compact subset of $\mathbb{R}^{n}$, and $f: K \rightarrow \mathbb{R}^{e}$ be a continuous function mapping. Then for an arbitrary $\varepsilon>0$, there exist an integer $N>0$, an $m \times N$ matrix $A$ and an $N$ dimensional vector $\theta$ such that 
\begin{eqnarray*}
\max_{x \in K} \vert f(x) - A \sigma(Bx + \theta)\vert < \epsilon, 
\end{eqnarray*}
holds where $\sigma: \mathbb{R}^{N} \mapsto \mathbb{R}^{N}$ is a sigmoid mapping defined by
\begin{eqnarray*}
\sigma('(u_{1}, \cdots, u_{N})) = '(\sigma(u_{1}), \cdots, \sigma(u_{N})).
\end{eqnarray*}
\end{lemma}
\begin{lemma}\label{lemma2}
Let $K$ be any compact subset of $\mathbb{R}^{d}$. Let $f$ and $\tilde{f}$ be two continuously differentiable functions on $\mathbb{R}^{d+e}$. Then it follows that
\begin{eqnarray*}
\vert\vert G_{f} - G_{\tilde{f}} \vert\vert_{\infty, K} < C\vert\vert f - \tilde{f}\vert\vert_{\infty, K}  , 
\end{eqnarray*}
where $C$ is a constant depending on the $\triangledown f$ and $N$, i.e.
\begin{eqnarray}\label{eqn_C}
C = \begin{cases}
\frac{C_{1}^{N} - 1}{C_{1} - 1},& \text{if }C_{1} \neq 1; \\
N, & \text{if }C_{1}= 1.
\end{cases}
\end{eqnarray}
\begin{eqnarray*}
C_{1}:= \sup_{(x_{1}, \cdots, x_{N}) \in K}\max_{k = 1}^{N} \vert\vert \triangledown_{o} f(x_{k}, o_{k})\vert\vert.  
\end{eqnarray*}
\end{lemma}
\begin{proof}
As $f$ and $\tilde{f}$ are continuous functions and $K$ is compact, then the image of $G_{f}$ and $G_{\tilde{f}}$ for any $(x_{1}, \cdots, x_{N}) \in K$ are compact. Let $(h_{i})_{i = 1}^{N}$ and $(\tilde{h}_{i})_{i = 1}^{N}$ denote $G_{f}$ and $G_{\tilde{f}}$ evaluated at $(x_{1}, x_{2}, \cdots, x_{N})$ respectively. Then we have 
\begin{eqnarray*}
h_{i+1} = f(x_{i+1},  h_{i}) \text{ and } \tilde{h}_{i+1}  = \tilde{f}(x_{i+1}, \tilde{h}_{i}). 
\end{eqnarray*}
Then it follows that
\begin{eqnarray*}
&&\vert\vert  h_{i+1}- \tilde{h}_{i+1}   \vert\vert = \vert\vert f(x_{i+1}, h_{i}) - \tilde{f}(x_{i+1}, \tilde{h}_{i}) \vert\vert\\
&\leq& \vert\vert f(x_{i+1}, h_{i}) - f(x_{i+1}, \tilde{h}_{i}) \vert\vert  +\\&& \vert\vert f(x_{i+1}, \tilde{h}_{i} ) - \tilde{f}(x_{i+1}, \tilde{h}_{i}) \vert\vert\\
&\leq& \vert\vert f(x_{i+1}, h_{i}) - \tilde{f}(x_{i+1}, \tilde{h}_{i}) \vert\vert +\\&& \sup_{x\in K }\vert\vert Df(x, h)\vert\vert \vert\vert h_{i} - \tilde{h}_{i}\vert\vert, 
\end{eqnarray*}
which shows the recursive relation of $\vert\vert  h_{i}- \tilde{h}_{i} \vert\vert$.\\
It is easy to check that if $a_{i+1} \leq C_{0}+ C_{1} a_{i}$ with $a_{0} = 0$, it implies that 
\begin{eqnarray*}
a_{i} \leq \begin{cases} \frac{C_{1}^{i} - 1}{C_{1} -1}C_{0}, \text{ if } C_{1} \neq 1\\
iC_{0}, \text{ if }C_{1} = 1.
\end{cases}
\end{eqnarray*}
Therefore using the above inequality when $a_{i} = \vert\vert h_{i} - \tilde{h}_{i}\vert\vert$, $C_{0} = \max_{x \in K} \vert\vert f-\tilde{f}\vert\vert$, it follows that
\begin{eqnarray*}
 \vert\vert h_{i} - \tilde{h}_{i}\vert\vert \leq C \vert\vert f-\tilde{f}\vert\vert_{\infty}, 
\end{eqnarray*}
and so does 
\begin{eqnarray*}
 \vert\vert G_{f} - G_{\tilde{f}}\vert\vert_{\infty, K} \leq C \vert\vert f-\tilde{f}\vert\vert_{\infty, K}, 
\end{eqnarray*}
where $C$ is defined by Equation \eqref{eqn_C}.
\end{proof}
\subsection{Proof of Theorem \ref{UniversalityLogsigRNN}}
First of all, let us give the intuition of the proof of Theorem \ref{UniversalityLogsigRNN}. There is a remarkable similarity between the recursive structure of an RNN $\mathcal{R}_{\sigma}$ and the one defined by the numerical Taylor approximation to solutions of SDEs $\hat{Y}_{u_{N}}$(Figure \ref{SDE_sol_vs_RNN}). It is noted that the numerical approximation of the solution $\hat{Y}_{u_{k+1}}^{\mathcal{D}, M}$ represented in Equation \eqref{RecursiveStructureY} depends on $\hat{Y}_{u_{k}}^{\mathcal{D}, M}$ and the log-signature $l_{k}$. 

\begin{figure}[t]
	\centering
	\includegraphics[width= 0.5\textwidth]{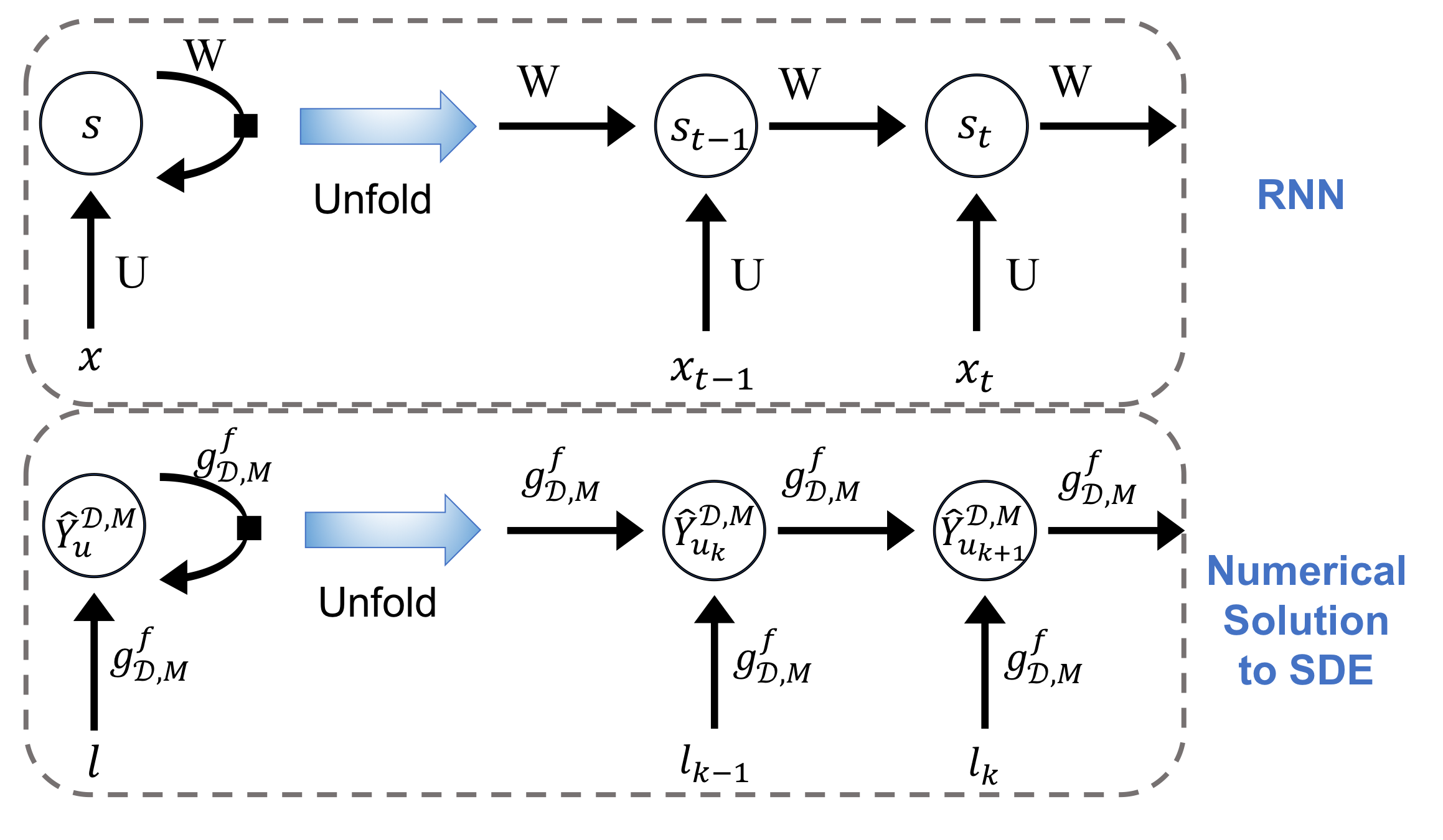}
	\caption{The shared recursive structure of numerical approximation of the solution $\hat{Y}_{u_{N}}$ and the RNN $\mathcal{R}_{\sigma}$.
	}\label{SDE_sol_vs_RNN}
\end{figure}

we use $G$ to denote the common recursive structure between those two. Specifically, for any given function $\tilde{f}: \mathbb{R}^{d+e} \rightarrow \mathbb{R}^{e}$, define $G_{\tilde{f}}:=G_{\tilde{f}, o_{1}, N}: \mathbb{R}^{N \times d} \rightarrow \mathbb{R}^{N \times e}$ as follows:
\begin{eqnarray*}
(x_{1}, \cdots, x_{N}) \mapsto (o_{1}, \cdots, o_{N}),
\end{eqnarray*}
where $o_{t+1} = \tilde{f}(x_{t+1}, o_{t}), \forall t \in \{1, \cdots, N-1\}$. As $o_{1}$ is set to be the same and $N$ is fixed in Theorem \ref{UniversalityLogsigRNN} , we skip the subscript $o_{1}$ and $N$ in the notation $G_{\tilde{f}}$.

On the one hand, when $\tilde{f}(x, s) := A\sigma(Ux +Ws)$, where $A$ is a matrix of dimension $d \times e$, $x \in \mathbb{R}^{d}$ and $s \in \mathbb{R}^{e}$, then $G_{\tilde{f}}$ is the RNN equipped with the activation function $\sigma$, denoted by $\mathcal{R}(\vert \Theta)$; on the other hand, the numerical solution to SDE $\hat{Y}_{u_{N}}$ is $G_{g_{\mathcal{D}, M}^{f}}$. Therefore the error $E_{2}$ is the norm of the difference between $G_{A\sigma(Ux +W)}$ and $G_{g_{\mathcal{D}, M}^{f}}$, i.e. 
\begin{eqnarray}\label{dif_G_f}
\hat{Y}_{u_{N}} - \mathcal{R}_{\sigma}((l_{k})_{k = 1}^{N} =\left( G_{g_{\mathcal{D}, M}^{f}} - G_{A\sigma(Ux +W)}\right) ((l_{k})_{k = 1}^{N})
\end{eqnarray}

\begin{proof}
By the triangle inequality, it holds that
\begin{eqnarray}
&&\vert\vert Y_{T} - \mathcal{R}_{\sigma}((l_{k})_{k = 1}^{N} \vert \Theta) \vert\vert \leq  \nonumber\\
&&\underset{E_{1}}{\underbrace{\vert\vert Y_{T} - \hat{Y}_{u_{N}} \vert\vert}}+ \underset{E_{2} }{\underbrace{\vert\vert \hat{Y}_{u_{N}} - \mathcal{R}_{\sigma}((l_{k})_{k = 1}^{N} \vert \Theta)\vert\vert}}.
\end{eqnarray}
By Global Approximation Theorem (Theorem \ref{GlobalApproxThm2}), $E_{1}$ can be arbitrarily small by setting $\Delta \mathcal{D}$ sufficiently small and truncation degree of the log-signature $M$ sufficiently large. 

The universality of the Logsig-RNN model is reduced to control the error $E_{2}$, which is the difference between $G_{A\sigma(Ux +W)}$ and $G_{g_{\mathcal{D}, M}^{f}}$ (Equation \ref{dif_G_f}). 
Lemma \ref{lemma1} ensures that the shallow neural network can approximate any continuous function uniformly well while Lemma \ref{lemma2} demonstrates the continuity of the map $\tilde{f} \mapsto G_{\tilde{f}}$. Combining both lemmas we are able to show that $E_{2}$ can be arbitrarily small provided that $\Delta \mathcal{D}$ sufficiently small and degree $M$ sufficiently large. 
\end{proof}

\section{Numerical Examples}
In this section, we add one more empirical data experiment on the UCI pen-digit recognition data. Moreover, we provide the implementation details of our method for the NTURGB+D 120 action data and Chalearn 2013 gesture data.
\subsection{UCI Pen-Digit Data}
In this subsection, we apply the Logsig-RNN algorithm on the UCI sequential pen-digit data\footnote{\url{https://archive.ics.uci.edu/ml/datasets/Pen-Based+Recognition+of+Handwritten+Digits}}. In Table \ref{tab:pendigit_dropout_result}, the Logsig-RNN with $M = 4$ and $N = 4$ achieves the accuracy $97.88\%$ in the testing data compared with $95.80\%$ of RNN$_{0}$. In addition, the training time of the Logsig-RNN takes $30\%$ of RNN$_{\mathcal{D}}$ and $3\%$ of the training time of RNN$_0$.
\begin{figure}[!ht]
\centering
\includegraphics[width=0.5\linewidth]{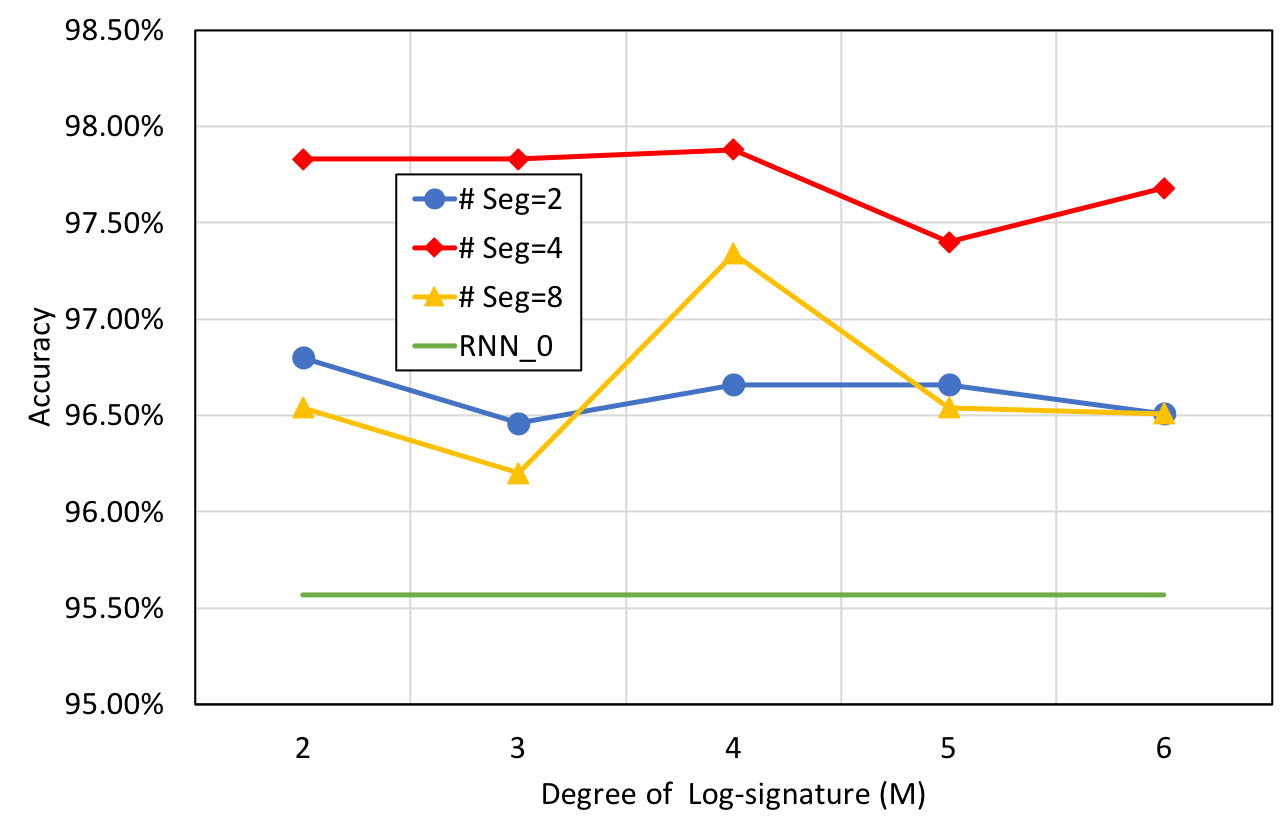}
\caption{The accuracy comparison of Logsig-RNN in the testing set.}
\end{figure}

%\subsubsection{Robustness to missing data}
\textbf{Robustness to missing data and change of sampling frequency}  To mimic the missing data case, we randomly throw a certain portion of points for each sample, and evaluated the trained models of Logsig-RNN and RNN$_{0}$ to the new testing data. Table \ref{tab:pendigit_dropout_result} shows that our proposed method outperforms the other methods significantly for the missing data case. Figure \ref{fig:pendigit_downsampling_result} shows the robustness of trained Logsig-RNN model for the down-sampled test data. Here the test data is down-sampled to that of length $8$ provided in UCI data. For $M=4$, the accuracy of testing data is above $80\%$, which is about $4$ times of that of the baseline methods.

\begin{table}[!htbp]
  \centering
  \scalebox{0.8}{
    \begin{tabular}{|c|c|c|c|c|c|c|}
    \hline
    \diagbox[width=5em]{r}{M} & 2     & 3     & 4  &5&6   & RNN$_0$ \\
    \hline
    0\%   & \cellcolor[rgb]{ .439,  .678,  .278}97.83\% & \cellcolor[rgb]{ .439,  .678,  .278}97.83\% & \cellcolor[rgb]{ .439,  .678,  .278}97.88\% & \cellcolor[rgb]{ .573,  .816,  .314}97.40\% & \cellcolor[rgb]{ .573,  .816,  .314}97.68\% & \cellcolor[rgb]{ 1,  .745,  .051}95.80\% \\
    \hline
    10\%  & \cellcolor[rgb]{ .573,  .816,  .314}97.63\% & \cellcolor[rgb]{ .439,  .678,  .278}97.77\% & \cellcolor[rgb]{ .776,  .878,  .706}97.14\% & \cellcolor[rgb]{ 1,  .851,  .4}96.88\% & \cellcolor[rgb]{ .573,  .816,  .314}97.60\% & \cellcolor[rgb]{ 1,  .514,  .384}40.91\% \\
    \hline
    20\%  & \cellcolor[rgb]{ 1,  .851,  .4}96.74\% & \cellcolor[rgb]{ .776,  .878,  .706}97.06\% & \cellcolor[rgb]{ 1,  .851,  .4}96.68\% & \cellcolor[rgb]{ 1,  .745,  .051}95.85\% & \cellcolor[rgb]{ .776,  .878,  .706}97.06\% & \cellcolor[rgb]{ 1,  .412,  .306}37.28\% \\
    \hline
    30\%  & \cellcolor[rgb]{ 1,  .745,  .051}95.99\% & \cellcolor[rgb]{ 1,  .745,  .051}95.40\% & \cellcolor[rgb]{ 1,  .745,  .051}95.17\% & \cellcolor[rgb]{ 1,  .745,  .051}95.03\% & \cellcolor[rgb]{ 1,  .745,  .051}95.77\% & \cellcolor[rgb]{ 1,  0,  0}32.65\% \\
    \hline
    \end{tabular}}%
    \caption{The accuracy of the modified testing set using different missing data rate ($r$). Here $N=4$.}
  \label{tab:pendigit_dropout_result}%
\end{table}%

\begin{figure}[!ht]
\centering
\includegraphics[width=0.5\linewidth]{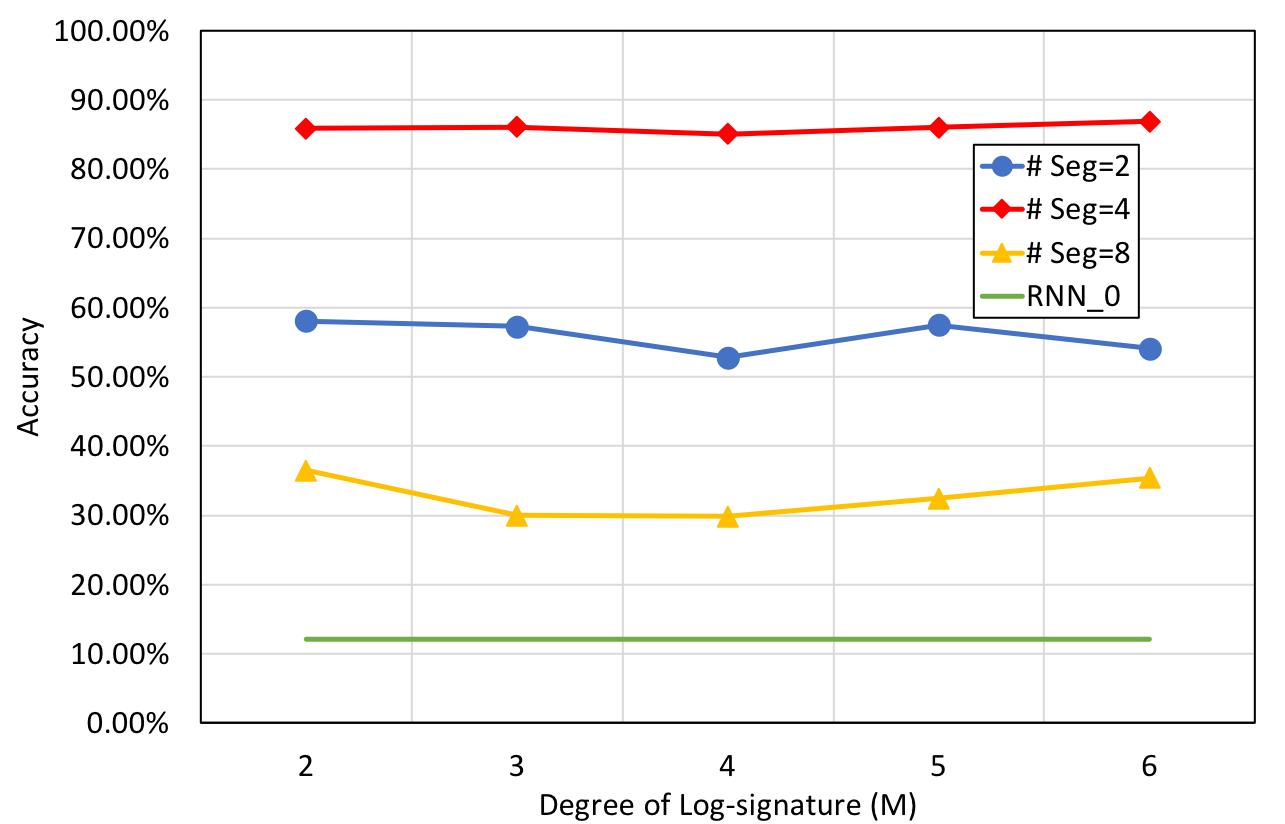}
   \caption{Validation of the trained models on the down-sampled dataset. The accuracy of RNN$_{\mathcal{D}}$ is below $13.5\%$}
\label{fig:pendigit_downsampling_result}
\end{figure}

\subsection{Action Recognition: NTU RGB+D 120 Data}
In this subsection, we provide the network architecture and implementation details of the LP-Logsig-RNN model for the NTU RGB+D 120 dataset. In Tabel~\ref{actionarche}, it displays the path transform layers is mainly composed with two Conv2D layers in \cite{zhu2016co} followed by one Conv1D layer. Before the LSTM layer, the starting frame of each segment of Conv1D output is incorporated into the output of Logsig layer. DropOut layer is applied after the LSTM layer with dropout rate 0.8.
The number of hidden neurons of LSTM layer is 256.
\subsection{Gesture Recognition: Chalearn 2013}
We provide the implementation details of the PT-Logsig-RNN model depicted in Figure~\ref{LP_Logsig_RNN} for the Chalearn 2013 data. The skeletons are pre-processed by first subtracting the central joint, which is the average position of all joints in one sample. Then we normalize the data and sample all clips to 39 frames by linear interpolation and uniform sampling. The path transform layers are composed of two Conv2D layers followed by a Conv1D layer, a Time-incorporated layer and an Accumulative Layer followed by the log-signature transformation with $d'=30$. We add DropOut layer to both of the embedding layer and the LSTM layer to avoid over-fitting, where the two dropout rate are $0.3$ and $0.5$ resp. The number of hidden neurons of LSTM layer is $128$. To make fair comparison, a DropOut layer is added to the benchmark RNN$_0$.
Three methods of data augmentation are used in the experiments. The first one is rotating coordinates along $x,y,z$ axis in range of $[-\pi/36,\pi/36]$, $[-\pi/18,\pi/18]$ and $[-\pi/36,\pi/36]$ respectively. The second one is randomly shifting the frame temporally in range of $[-5,5]$. The last one is adding Gaussian noise with standard deviation $0.001$ to joints coordinates.

\begin{table}[t]
\centering
\scalebox{0.75}{
\begin{tabular}{l|l|l}
\hline
\textbf{Layer} & \textbf{Output shape} & \textbf{Discription}\\
\hline
Input & $(72, 25, 6)$ & \\
Conv2D & $(72, 25, 32)$ & Kernel size=$1\times 1\times 32$ \\
Conv2D & $(68, 25, 16)$ & Kernel size=$5\times 1\times 16$\\
Conv1D & $(68, 40)$ & Kernel size=$1\times 400\times 40$\\
Logsig & $(32, 820)$ & $M=2, N=4$\\
Add Starting Points & $(32,860)$ & Starting points of Conv1D output\\
LSTM & $(32,256)$ & Return sequential output \\
Output & $120$ &\\
\hline
\end{tabular}}
\caption{Architecture of the action recognition model}
\label{actionarche}
\end{table}

\section{The Universality of Controlled Differential Equations}
\begin{definition}
Let $Y_{t}$ be the solution to the following equation
\begin{eqnarray}\label{difEquation}
dY_{t} = V(Y_{t}) dX_{t}, Y_{0} = y_{0},
\end{eqnarray}
where $X: [0, T] \rightarrow = E$ of finite $1$-variation, and $V \in \mathcal{C}_{b}^{\infty}$ and $Y: [0, T] \rightarrow W:= \mathbb{R}^{o} $.\\
Let $I_{V}$ denote the solution map, i.e. for any $X$ is in the admissible set,
\begin{eqnarray*}
I_{V}: (X, y_{0}) \mapsto Y.
\end{eqnarray*}
\end{definition}
\begin{theorem}\label{ControlEqnThm}
The linear functional on the solution map on the path space defined in Equation (\ref{difEquation}) restricted to the linear vector fields forms an algebra with componentwise multiplication. 
\end{theorem}
\begin{proof}
Suppose that $Y$ and $\tilde{Y}$ be the solution to the following equation driven by $X$, i.e.
Let $Y_{t}$ be the solution to the following equation
\begin{eqnarray}\label{difEquation}
&&dY_{t} = L(Y_{t}) dX_{t}, Y_{0} = y_{0}\nonumber\\
&&d\tilde{Y}_{t} = \tilde{L}(\tilde{Y}_{t}) dX_{t}, \tilde{Y}_{0} = \tilde{y}_{0},
\end{eqnarray}
where $L$ and $\tilde{L}$ are two linear vector fields.

For any two basis $w_{i}^{*}$ and $w_{j}^{*}$ of $W^{*}$, 
\begin{eqnarray}
\langle w_{i}^{*}, Y_{t}\rangle \langle w_{j}^{*}, \tilde{Y}_{t}\rangle=Y^{(i)}_{t}\tilde{Y}_{t}^{(j)}
\end{eqnarray}
and thus
\begin{eqnarray}
&& d (Y^{(i)}_{t}\tilde{Y}_{t}^{(j)}) = Y^{(i)}_{t}d\tilde{Y}_{t}^{(j)} + \tilde{Y}_{t}^{(j)}dY^{(i)}_{t}\\
& = & \langle w_{j}^{*},  Y^{(i)}_{t} \tilde{L}(\tilde{Y}_{t})dX_{t}\rangle + \langle w_{i}^{*},  \tilde{L}^{(j)}_{t} L(Y_{t})dX_{t}\rangle\nonumber,
\end{eqnarray}
which can be rewritten as follows: there exist $l_{1}, l_{2} \in L(W, L(E, W))$, such that
\begin{eqnarray*}
d (Y^{(i)}_{t}\tilde{Y}_{t}^{(j)}) =\langle w_{j}^{*}, l_{1}( \tilde{Y}_{t}\otimes Y_{t}) dX_{t}\rangle + \langle w_{i}^{*}, l_{2}( \tilde{Y}_{t}\otimes Y_{t}) dX_{t}\rangle. 
\end{eqnarray*}
Therefore,  $\langle w_{i}^{*}, Y_{t}\rangle \langle w_{j}^{*}, \tilde{Y}_{t}\rangle$ can be rewritten as a linear functional on $Z_{t} = Y_{t} \otimes \tilde{Y}_{t}$, which is the solution to a linear controlled differential equation driven by $X$, i.e.
\begin{eqnarray*}
dZ_{t} = L_{Y, \tilde{Y}}(Z_{t})dX_{t},
\end{eqnarray*}
where $L_{Y, \tilde{Y}}: W \otimes W \rightarrow L(E, W \otimes W)$.
\end{proof}
Let $\mathcal{C}_{1}(V_{1}(J, E), W)$ denote the space of continuous functionals on $V_{1}(J, E)$ taking values in $W$, which are invariant w.r.t time parameterization. 
\begin{theorem}[Universality of the linear controlled differential equations]
The linear functionals on the solution map on the path space defined in Equation (\ref{difEquation}) are dense in the space $\mathcal{C}_{1}(V_{1}(J, E), W)$.
\end{theorem}
\begin{proof}
Let $L$ be a trivial vector field and $y_{0} = 1$. Then $I_{V}\equiv 1$. By Theorem \ref{ControlEqnThm}, the linear functions on the solution map form an algebra under the multiplication. By the Stone-Weierstrass Theorem, the proof is complete.   
\end{proof}
The set of the solution map driven by potential non-linear vector fields include all the solution driven by linear vector fields. Therefore we establish the universality of linear functionals on solution to controlled differential equations provided vector fields under the regularity condition. 
\end{appendices}

%% file: arxiv_version_v2.bbl
\begin{thebibliography}{10}\itemsep=-1pt

\bibitem{bishwal2007parameter}
J.~P. Bishwal.
\newblock {\em Parameter estimation in stochastic differential equations}.
\newblock Springer, 2007.

\bibitem{black1973pricing}
F.~Black and M.~Scholes.
\newblock The pricing of options and corporate liabilities.
\newblock {\em Journal of political economy,}, 81(3):637--654, 1973.

\bibitem{chen2018neural}
T.~Q. Chen, Y.~Rubanova, J.~Bettencourt, and D.~K. Duvenaud.
\newblock Neural ordinary differential equations.
\newblock In {\em NIPS}, pages 6571--6583, 2018.

\bibitem{cox1976valuation}
J.~C. Cox and S.~A. Ross.
\newblock The valuation of options for alternative stochastic processes.
\newblock {\em Journal of financial economics,}, 3(1-2):145--166, 1976.

\bibitem{donoho2000high}
D.~L. Donoho et~al.
\newblock High-dimensional data analysis: The curses and blessings of
  dimensionality.
\newblock {\em AMS math challenges lecture,}, 1(2000):32, 2000.

\bibitem{weinan2017proposal}
W.~E.
\newblock A proposal on machine learning via dynamical systems.
\newblock {\em Comm. in Math. and Stat.,}, 5(1):1--11, 2017.

\bibitem{Escalera2013MultimodalGR}
S.~Escalera, J.~Gonz{\`a}lez, X.~Bar{\'o}, M.~Reyes, O.~Lopes, I.~Guyon,
  V.~Athitsos, and H.~J. Escalante.
\newblock Multi-modal gesture recognition challenge 2013: dataset and results.
\newblock In {\em ICMI}, 2013.

\bibitem{foster2019optimal}
J.~Foster, T.~Lyons, and H.~Oberhauser.
\newblock An optimal polynomial approximation of brownian motion.
\newblock {\em arXiv preprint arXiv:1904.06998}, 2019.

\bibitem{friz2015physical}
P.~Friz, P.~Gassiat, and T.~Lyons.
\newblock Physical brownian motion in a magnetic field as a rough path.
\newblock {\em Transactions of the American Mathematical Society},
  367(11):7939--7955, 2015.

\bibitem{friz2010multidimensional}
P.~Friz and N.~Victoir.
\newblock {\em Multidimensional Stochastic Processes as Rough Paths: Theory and
  Applications}.
\newblock Cambridge Studies in Advanced Mathematics. 2010.

\bibitem{funahashi1993approximation}
K.-i. Funahashi and Y.~Nakamura.
\newblock Approximation of dynamical systems by continuous time recurrent
  neural networks.
\newblock {\em Neural networks,}, 6(6):801--806, 1993.

\bibitem{gardiner1985handbook}
C.~W. Gardiner et~al.
\newblock {\em Handbook of stochastic methods}, volume~3.
\newblock Springer Berlin, 1985.

\bibitem{graham2013sparse}
B.~Graham.
\newblock Sparse arrays of signatures for online character recognition.
\newblock {\em arXiv preprint arXiv:1308.0371}, 2013.

\bibitem{gyurko2013extracting}
L.~G. Gyurk{\'o}, T.~Lyons, M.~Kontkowski, and J.~Field.
\newblock Extracting information from the signature of a financial data stream.
\newblock {\em arXiv preprint arXiv:1307.7244}, 2013.

\bibitem{UniquenessOfSignature}
B.~Hambly and T.~Lyons.
\newblock Uniqueness for the signature of a path of bounded variation and the
  reduced path group.
\newblock {\em Annals of Mathematics,}, 171(1):109--167, 2010.

\bibitem{7784788}
J.~{Hu}, W.~{Zheng}, J.~{Lai}, and J.~{Zhang}.
\newblock Jointly learning heterogeneous features for rgb-d activity
  recognition.
\newblock {\em IEEE TPAMI,}, 39(11):2186--2200, Nov 2017.

\bibitem{ReizensteinIhesis2018}
R.~Jeremy.
\newblock {\em Iterated-Integral Signatures in Machine Learning}.
\newblock PhD thesis, 2019.

\bibitem{ReizensteinIisignature2018}
R.~Jeremy and G.~Benjamin.
\newblock The iisignature library: efficient calculation of iterated-integral
  signatures and log signatures.
\newblock {\em arXiv preprint arXiv:1802.08252.}, 2018.

\bibitem{8306456}
Q.~{Ke}, M.~{Bennamoun}, S.~{An}, F.~{Sohel}, and F.~{Boussaid}.
\newblock Learning clip representations for skeleton-based 3d action
  recognition.
\newblock {\em IEEE TIP,}, 27(6):2842--2855, June 2018.

\bibitem{levin2013learning}
D.~Levin, T.~Lyons, and H.~Ni.
\newblock Learning from the past, predicting the statistics for the future,
  learning an evolving system.
\newblock {\em arXiv preprint arXiv:1309.0260}, 2013.

\bibitem{li2017lpsnet}
C.~Li, X.~Zhang, and L.~Jin.
\newblock Lpsnet: a novel log path signature feature based hand gesture
  recognition framework.
\newblock In {\em CVPR}, pages 631--639, 2017.

\bibitem{li2019skeleton}
C.~Li, X.~Zhang, L.~Liao, L.~Jin, and W.~Yang.
\newblock Skeleton-based gesture recognition using several fully connected
  layers with path signature features and temporal transformer module.
\newblock In {\em AAAI}, pages 8585--8593, 2019.

\bibitem{Liu_2019_NTURGBD120}
J.~Liu, A.~Shahroudy, M.~Perez, G.~Wang, L.-Y. Duan, and A.~C. Kot.
\newblock Ntu rgb+d 120: A large-scale benchmark for 3d human activity
  understanding.
\newblock {\em IEEE TPAMI}, 2019.

\bibitem{liufsnet}
J.~Liu, A.~Shahroudy, G.~Wang, L.-Y. Duan, and A.~C.~Kot.
\newblock Skeleton-based online action prediction using scale selection
  network.
\newblock {\em IEEE TPAMI}, 02 2019.

\bibitem{8101019}
J.~{Liu}, A.~{Shahroudy}, D.~{Xu}, A.~C. {Kot}, and G.~{Wang}.
\newblock Skeleton-based action recognition using spatio-temporal lstm network
  with trust gates.
\newblock {\em IEEE TPAMI,}, 40(12):3007--3021, Dec 2018.

\bibitem{10.1007/978-3-319-46487-9_50}
J.~Liu, A.~Shahroudy, D.~Xu, and G.~Wang.
\newblock Spatio-temporal lstm with trust gates for 3d human action
  recognition.
\newblock In {\em ECCV}, pages 816--833, 2016.

\bibitem{Liu2018SkeletonBasedHA}
J.~Liu, G.~Wang, L.~yu~Duan, K.~Abdiyeva, and A.~C. Kot.
\newblock Skeleton-based human action recognition with global context-aware
  attention lstm networks.
\newblock {\em IEEE TIP,}, 27:1586--1599, 2018.

\bibitem{liu2018recognizing}
M.~Liu and J.~Yuan.
\newblock Recognizing human actions as the evolution of pose estimation maps.
\newblock In {\em CVPR}, pages 1159--1168, 2018.

\bibitem{lu2017beyond}
Y.~Lu, A.~Zhong, Q.~Li, and B.~Dong.
\newblock Beyond finite layer neural networks: Bridging deep architectures and
  numerical differential equations.
\newblock {\em arXiv preprint arXiv:1710.10121}, 2017.

\bibitem{RoughPaths}
T.~Lyons, T.~L$\acute{e}$vy, and M.~Caruana.
\newblock {\em Differential Equation driven by Rough Paths}.
\newblock Springer, 2006.

\bibitem{lyons2014feature}
T.~Lyons, H.~Ni, and H.~Oberhauser.
\newblock A feature set for streams and an application to high-frequency
  financial tick data.
\newblock In {\em ACM International Conference on Big Data Science and
  Computing}, page~5, 2014.

\bibitem{lyons1998differential}
T.~J. Lyons.
\newblock Differential equations driven by rough signals.
\newblock {\em Revista Matem{\'a}tica Iberoamericana,}, 14(2):215--310, 1998.

\bibitem{merton1973theory}
R.~C. Merton et~al.
\newblock Theory of rational option pricing.
\newblock {\em Theory of Valuation,}, pages 229--288, 1973.

\bibitem{muller2011functional}
H.-G. M{\"u}ller, R.~Sen, and U.~Stadtm{\"u}ller.
\newblock Functional data analysis for volatility.
\newblock {\em Journal of Econometrics,}, 165(2):233--245, 2011.

\bibitem{ni2015multi}
H.~Ni.
\newblock A multi-dimensional stream and its signature representation.
\newblock {\em arXiv preprint arXiv:1509.03346}, 2015.

\bibitem{papavasiliou2011parameter}
A.~Papavasiliou, C.~Ladroue, et~al.
\newblock Parameter estimation for rough differential equations.
\newblock {\em The Annals of Statistics,}, 39(4):2047--2073, 2011.

\bibitem{reizenstein2018iisignature}
J.~Reizenstein and B.~Graham.
\newblock The iisignature library: efficient calculation of iterated-integral
  signatures and log signatures.
\newblock {\em arXiv preprint arXiv:1802.08252}, 2018.

\bibitem{reutenauer2003free}
C.~Reutenauer.
\newblock Free lie algebras.
\newblock In {\em Handbook of algebra}, volume~3, pages 887--903. Elsevier,
  2003.

\bibitem{nturgb}
A.~Shahroudy et~al.
\newblock Ntu rgb+d: A large scale dataset for 3d human activity analysis.
\newblock In {\em CVPR}, 06 2016.

\bibitem{silverman1996smoothed}
B.~W. Silverman et~al.
\newblock Smoothed functional principal components analysis by choice of norm.
\newblock {\em The Annals of Statistics,}, 24(1):1--24, 1996.

\bibitem{Wang2017ModelingTD}
H.~Wang and L.~Wang.
\newblock Modeling temporal dynamics and spatial configurations of actions
  using two-stream recurrent neural networks.
\newblock {\em CVPR,}, pages 3633--3642, 2017.

\bibitem{xie2018learning}
Z.~Xie, Z.~Sun, L.~Jin, H.~Ni, and T.~Lyons.
\newblock Learning spatial-semantic context with fully convolutional recurrent
  network for online handwritten chinese text recognition.
\newblock {\em IEEE TPAMI,}, 40(8):1903--1917, 2018.

\bibitem{yang2017leveraging}
W.~Yang, T.~Lyons, H.~Ni, C.~Schmid, L.~Jin, and J.~Chang.
\newblock Leveraging the path signature for skeleton-based human action
  recognition.
\newblock {\em arXiv preprint arXiv:1707.03993}, 2017.

\bibitem{zhu2016co}
W.~Zhu, C.~Lan, J.~Xing, W.~Zeng, Y.~Li, L.~Shen, and X.~Xie.
\newblock Co-occurrence feature learning for skeleton based action recognition
  using regularized deep lstm networks.
\newblock In {\em AAAI}, 2016.

\end{thebibliography}
